%% file: main_full.tex
\documentclass[sigconf]{aamas}

\usepackage{balance} %

\setcopyright{ifaamas}
\acmConference[AAMAS '23]{Proc.\@ of the 22nd International Conference
on Autonomous Agents and Multiagent Systems (AAMAS 2023)}{May 29 -- June 2, 2023}
{London, United Kingdom}{A.~Ricci, W.~Yeoh, N.~Agmon, B.~An (eds.)}
\copyrightyear{2023}
\acmYear{2023}
\acmDOI{}
\acmPrice{}
\acmISBN{}

\acmSubmissionID{1133}

\title[]{Coordination of Multiple Robots along Given Paths with Bounded Junction Complexity}

\author{Mikkel Abrahamsen}
\affiliation{
  \institution{University of Copenhagen, Denmark}
  \country{}%
  }
\email{miab@di.ku.dk}

\author{Tzvika Geft}
\affiliation{
  \institution{Tel Aviv University, Israel}
  \country{}%
  }
\email{zvigreg@mail.tau.ac.il}

\author{Dan Halperin}
\affiliation{
  \institution{Tel Aviv University, Israel}
  \country{}%
  }
\email{danha@tauex.tau.ac.il}

\author{Barak Ugav}
\affiliation{
  \institution{Tel Aviv University, Israel}
  \country{}%
  }
\email{barakugav@gmail.com}

\begin{abstract}
We study a fundamental NP-hard motion coordination problem for multi-robot/multi-agent systems:
We are given a graph $G$ and set of agents, where each agent has a given directed path in $G$.
Each agent is initially located on the first vertex of its path.
At each time step an agent can move to the next vertex on its path, provided that the vertex is not occupied by another agent.
The goal is to find a sequence of such moves along the given paths so that each reaches its target, or to report that no such sequence exists.
The problem models guidepath-based transport systems,
which is a pertinent abstraction for traffic in a variety of contemporary applications, ranging from train networks or Automated Guided Vehicles (AGVs) in factories, through computer game animations, to qubit transport in quantum computing.
It also arises as a sub-problem in the more general multi-robot motion-planning problem.

We provide a fine-grained tractability analysis of the problem by considering new assumptions and identifying minimal values of key parameters for which the problem remains NP-hard.
Our analysis identifies a critical parameter called \emph{vertex multiplicity} (VM), defined as the maximum number of paths passing through the same vertex.
We show that a prevalent variant of the problem, which is equivalent to Sequential Resource Allocation (concerning deadlock prevention for concurrent processes), is NP-hard even when VM is 3.
On the positive side, for VM $\le$ 2 we give an efficient algorithm that iteratively resolves cycles of blocking relations among agents.
We also present a variant that is NP-hard when the VM is 2 even when $G$ is a 2D grid and each path lies in a single grid row or column.
By studying highly distilled yet NP-hard variants, we deepen the understanding of what makes the problem intractable and thereby guide the search for efficient solutions under practical assumptions.
\end{abstract}


\newcommand{\BibTeX}{\rm B\kern-.05em{\sc i\kern-.025em b}\kern-.08em\TeX}

\renewenvironment{quote}
  {\list{}{\rightmargin=0.5cm \leftmargin=0.5cm}%
   \item\relax}
  {\endlist}

\hypersetup{
    colorlinks,
    linkcolor={red!50!black},
    citecolor={blue!50!black},
    urlcolor={blue!80!black}
}
\usepackage{thm-restate}
\usepackage{titlesec}
\titleformat*{\subsubsection}{\LARGE\filcenter\bfseries}
\usepackage[utf8]{inputenc}
\usepackage{dblfloatfix} 
\usepackage[T1]{fontenc}    %
\usepackage{url}            %
\usepackage{booktabs}       %
\usepackage{amsfonts}       %
\usepackage{nicefrac}       %
\usepackage{microtype}      %
\usepackage{amsmath,amsthm}
\usepackage{cleveref}
\usepackage{subcaption}
\usepackage{enumitem}
\usepackage{wrapfig,graphicx}
\usepackage{xcolor}
\usepackage{comment}
\usepackage{mathtools}
\usepackage{float}
\usepackage[linesnumbered,ruled,noend,noline]{algorithm2e}
\usepackage[font=small,labelfont=bf]{caption}
\usepackage [english]{babel}
\usepackage [autostyle, english = american]{csquotes}
\usepackage{needspace}
\MakeOuterQuote{"}

\usepackage{graphics} %
\usepackage{epsfig} %
\usepackage{wrapfig}
\usepackage{authblk}

\usepackage{tikz}
\usetikzlibrary{shapes}
\usepackage{tikz-network}
\makeatletter
\tikzset{my loop/.style = {to path={
  \pgfextra{}
  [looseness=12,min distance=10mm]
  \tikz@to@curve@path},font=\sffamily\small
  }}

\usepackage{multirow}
\usepackage{booktabs}
\usepackage{pdfrender}
\newcommand*{\boldcheckmark}{%
  \textpdfrender{
    TextRenderingMode=FillStroke,
    LineWidth=.5pt, %
  }{\checkmark}%
}

\SetKwFor{Loop}{Loop}{}{EndLoop}

\newcommand{\commentblock}[1]{}
\DeclareMathOperator{\nextb}{\textit{next}}
\DeclareMathOperator{\prevb}{\textit{prev}}
\DeclareMathOperator{\head}{\textit{head}}
\DeclareMathOperator{\tail}{\textit{tail}}
\newcommand{\VM}{\mathop{\mathsf{VM}}}

\begin{document}
\input{mrmp_macros.tex}

\pagestyle{fancy}
\fancyhead{}

\maketitle

\section{Introduction}
We study the problem of coordinating the motion of a fleet of robots/agents\footnote{We interchangeably use the terms \emph{agents} and \emph{robots} in this work. In the terminology of the motion-planning literature, robots are typically used when moving in continuous domains, and agents (or pebbles, among others) are used for motion on graphs. The distinction is sometimes blurred since often continuous motion-planning problems are reduced to motion planning of agents on graphs.} with assigned paths.

The problem arises in the context of guidepath-based vehicles, such as
Automated Guided Vehicles and overhead monorail systems
that are used in industrial environments~\cite{DBLP:journals/tac/ReveliotisR10}.
Such environments are typically highly structured and constrain the vehicles to move along predefined paths in a centrally controlled manner.
A crucial component of such systems is ensuring \emph{liveness}, which is the ability of the vehicles to complete their assigned tasks and perform similar future tasks.
Liveness can be lost due to \textit{deadlocks}, which can arise due to the fact that certain vehicle motions are irreversible, for example, when a vehicle cannot move backward on a railway.
Before entering a new state, such as giving a vehicle a new assigned path, it is desirable to check if liveness can be preserved.
In general, this check amounts to determining the existence of a sequence of motions that allows all agents to complete their current trips, which in turn boils down to solving our problem.
This problem of checking state liveness is known as liveness-enforcing supervision and has received interest from the Discrete Event Systems (DES) community~\cite{DBLP:conf/med/ReveliotisM19}.

Another motivation comes from the world of smart transportation.
Recent years have demonstrated that operating autonomous vehicles in mixed traffic/urban areas remains highly challenging and is therefore unlikely to prevail soon.
A more viable setting for operating them safely is using dedicated infrastructure (e.g., guideways, rails, or dedicated lanes), 
which is simpler due to limited interaction with human drivers, pedestrians, and obstructions.
Unlike regular road networks, such infrastructures constrain vehicles to move on a limited set of paths, which ultimately lends itself to our setting.
Such constrained autonomous systems are expected to evolve beyond simple topologies and fixed schedules (e.g., in airport shuttles)~\cite{DBLP:journals/networks/KaspiRU22a} and hence demand more complex motion coordination.
For example, to cater for increased demand, which will also be flexible (e.g., as a result of on-demand service), a system's efficiency might be increased by allowing vehicles heading in opposite directions to use the same path segments.
Indeed, developing algorithms for the special structure of dedicated infrastructure transport systems has been recently highlighted as a research direction in urban mobility
and logistics~\cite{DBLP:journals/networks/KaspiRU22a}.

Our problem belongs in the wider context of \emph{Multi-Robot Motion Planning} (MRMP). %
In MRMP, instead of having the whole path specified, we are given only the start and target for each robot/agent, and the goal is to find a collision-free motion that brings all the robots to their targets. %
The general problem has been shown to be PSPACE-hard in various planar settings~\cite{hopcroft1984complexity, DBLP:journals/tcs/HearnD05, DBLP:journals/ijrr/SoloveyH16, DBLP:conf/fun/BrockenHKLS21}.
The problem is also NP-hard even when all the robots move one by one~\cite{geft2021tractability}.
Relaxations involving assumptions on the spacing between robots in
their start and target placements have been introduced in order to make the problem tractable~\cite{DBLP:journals/tase/AdlerBHS15, SolomonHalperin2018, DBLP:conf/rss/SoloveyYZH15}. %
The discrete counterpart of MRMP, where agents move on a graph, is solvable in polynomial time~\cite{DBLP:conf/focs/KornhauserMS84, DBLP:conf/wafr/YuR14}.
However, when optimal solutions  are sought, e.g., with respect to a time or distance objective, the problem becomes NP-hard, even on 2D grid graphs~\cite{DBLP:journals/ral/BanfiBA17, DBLP:journals/siamcomp/DemaineFKMS19, DBLP:conf/atal/GeftH22}.
This intractability has been recently tackled by approximation algorithms~\cite{DBLP:journals/siamcomp/DemaineFKMS19, DBLP:journals/arobots/Yu20, DBLP:journals/ral/WangR20}.
We remark that there is a tremendous body of work on MRMP variants, for which it is impossible to do justice here.%

Although MRMP has more freedom than our problem due to not constraining paths a priori, MRMP with given paths (MRMP-GP for short)
is useful as a subroutine for solving MRMP.
A common paradigm for solving MRMP, known as \emph{decoupled planning}, is to first plan the path of each robot without taking other robots into account. Then, some form of coordination between robots given their individual paths follows, such as adjusting the robots' speeds along paths.
Early works focused on variants of the problem for a constant number of robots~\cite{kant1986toward}.
In~\cite{DBLP:conf/icra/ODonnellL89}, the \textit{coordination diagram}, which represents placements along each robot's path at which mutual collisions might occur, was used for two robots.
The diagram has later been generalized for multiple robots~\cite{DBLP:conf/icra/LaValleH96} and has since been used to coordinate their motion along assigned paths~\cite{DBLP:journals/trob/SimeonLL02, olmi2010coordination}.
However, the size of the diagram has a worst-case exponential dependency on the number of robots.
A different flavor of works focuses on the \textit{execution} of paths (more precisely, trajectories) that have already been planned.
In practice, it is hard to guarantee that a robot will not deviate from a planned trajectory.
As robots become unexpectedly delayed, such as due to external interference, they might no longer be able to follow the initially valid plan.
Therefore, there is a perpetual need to verify and coordinate the motion of robots while they move along planned paths (possibly replanning them).
Some works of this flavor include~\cite{DBLP:conf/iros/CapGF16, DBLP:conf/flairs/CoskunOV21, DBLP:journals/corr/abs-2010-05254,DBLP:journals/ral/HonigKTDA19}.

Despite sustained interest in the problem, the first NP-hardness results explaining the lack of efficient algorithms appeared, to our knowledge, only relatively recently~\cite{DBLP:journals/tac/ReveliotisR10, DBLP:conf/med/ReveliotisM19}.
These results have been presented as part of a line of work on deadlock prevention in resource allocation problems~\cite{reveliotis2006real}, affiliated with the DES community, which have wide applicability in various automation scenarios~\cite{DBLP:journals/ram/Reveliotis15}.
Such problems involve allocating a finite set of reusable
resources to a set of concurrently executing processes, where each process can only be executed by acquiring resources in a certain order.
Our problem can be seen as being equivalent to the simplest class of resource allocation problems, known as Linear Single-Unit Resource Allocation Systems (L-SU-RAS), which is NP-hard~\cite{DBLP:journals/tac/ReveliotisR10, reveliotis2006real}.
In the terms of L-SU-RAS, each robot emulates a process that needs to "acquire" one location at a time from a sequence of locations, each of which can host only a single robot at a time.

We observe the following undesirable properties of previous hardness constructions:
(i) Unbounded vertex multiplicity, i.e., there is a "congested" location that has to be visited by an unbounded number of robots.
(ii) There are path segments that have to be traversed in opposite directions, i.e., an inherent potential for a head-on collision exists.
(iii) The robot's paths are not the shortest between their endpoints. In particular, robots have to visit the same vertex multiple times along their path (in proofs for the planar case).
We summarize the previous hardness results in $\Cref{tab:comp}$.
Arguably, these properties do not necessarily represent real systems, which prompts new analysis for cases where they do not hold.

\paragraph{Contribution.}
In this work, we perform a fine-grained complexity study of the problem of coordinating the motion of robots along fixed paths.
We consider two parameters: vertex multiplicity and the path shape complexity, which is the maximum number of turns made on the grid (exact definitions are provided in \Cref{sec:prob_def}).
We consider two main problem variants: (1) a prominent variant that corresponds to the L-SU-RAS resource allocation problem, and (2) a lesser-studied variant where the aforementioned property (ii) of opposite direction paths does not hold. For each variant, we identify the critical values of the parameters for which the problem remains NP-hard, such that below this value the problem is efficiently solvable.
By that, we establish a sharper boundary between negative and positive results.

Our main positive result is for variant (1) where we present an efficient algorithm solving the problem for any graph such that the vertex multiplicity is 2.
By that, we expand the class of efficiently solvable instances.
At the heart of our solution is an iterative procedure for resolving cycles
of blocking relations among agents. For each such cycle, we construct a special type of  a directed graph, which we call a \emph{graph composed of paths}. We repeatedly contract this graph until it reaches a problem-equivalent simplified and irreducible state in which it is easy to determine whether the blocking can be untangled or otherwise a deadlock is detected and the instance has no solution. The algorithm runs in time linear in the total lengths of the given paths.

On the negative side, we show that for a vertex multiplicity of~3, the problem is NP-hard.
Furthermore, for variant (2) we show that the problem is NP-hard for agents moving along straight paths on a 2D grid graph where the vertex multiplicity is 2, which is considerably more restricted than previous hardness results.

Our fine-grained complexity analysis echoes recent calls for deepening the understanding of what makes multi-agent motion coordination problems hard~\cite{DBLP:conf/atal/SalzmanS20, gordon2021revisiting, DBLP:conf/atal/EwingRKSA22}.
Such an improved understanding can guide the search towards improved algorithms and new efficiently solvable problem variants.
We discuss potential developments in this spirit in the conclusion.

\section{Problem Definition and Assumptions}
\label{sec:prob_def}
In this section, we formally define the problem that we study, along with the assumptions and parameters we consider.

\begin{table}
\setlength{\tabcolsep}{2.5pt}
\caption{Comparison to previous hardness results.
For exact definitions of problem variants and parameters see~\Cref{sec:prob_def}.
A check mark in the last column ("Shortest paths?") indicates that each given path in the construction is the shortest between its start and target vertices.
We use "min" to indicate that a parameter is the minimum value below which the problem variant is efficiently solvable.}
\label{tab:comp}
\centering
\begin{tabular}{lcccccc}
\midrule
\begin{tabular}[c]{@{}c@{}}Problem\\ variant \end{tabular} &
Paper& \begin{tabular}[c]{@{}c@{}}Graph\\ type\end{tabular} & \small \begin{tabular}[c]{@{}c@{}}Vertex\\ multiplicity \end{tabular}   & 
\small \begin{tabular}[c]{@{}c@{}}Max. \# turns\\ in path \end{tabular}  &  \small \begin{tabular}[c] {@{}c@{}}Shortest\\paths?\end{tabular} \\

 \midrule
 \multirow{5}{*}{NBT} & 
 \cite{lawley2001deadlock}     &  general & \small unbounded  &  n/a & \checkmark  \\
 & \cite{DBLP:conf/med/ReveliotisM19}  &  planar  &  \small unbounded &  n/a &             \\
 & \cite{DBLP:journals/tac/ReveliotisR10}  &  2D grid    & \small unbounded  &  >20 &           \\
  \cmidrule{2-6} %
 & \multirow{2}{*}{\textbf{ours}} & 
     \textbf{general}  & \textbf{ 3 (min) }  & \textbf{n/a}     & \boldcheckmark  \\
  &  & \textbf{2D grid}   & \textbf{ 4 }      & \textbf{1 (min) } &  \boldcheckmark   \\

\midrule
UNI   & \textbf{ours}& \textbf{2D grid} & \textbf{ 2 (min) }  & \textbf{ 0 (min) }& \boldcheckmark \\
     
\end{tabular}
\end{table}

\paragraph{Multi-robot motion planning with given paths (MRMP-GP).}
We are given a set $R$ of $n$ robots that operate in a workspace $W$, which is a finite undirected graph.
The vertices of $W$, denoted by $V(W)$, may also be called \emph{positions}.
Each robot $r \in R$ has a start vertex $\src{r} \in V(W)$ (also referred to as a \emph{source}) and a target vertex $\trg{r} \in V(W)$, and a path $\pi(r)=v_1,\ldots,v_\ell$, where $v_1=\src{r}, v_\ell=\trg{r}$ and $v_i v_{i+1}$
is an edge of $W$ for all $1 \le i < \ell$.
Note that since $\pi(r)$ is a path, each vertex only appears once in $\pi(r)$.

Each robot is initially located at its start vertex. MRMP-GP asks to find a \emph{motion plan} or \emph{solution}, which is a sequence of moves that bring each robot from its source to its target using its given path without inducing collisions with other robots.
A legal move consists of a single robot moving to the next vertex $v$ on its path, provided that no other robot is located at $v$.
No backward moves are allowed, i.e., a robot may not move to the previous vertex on its path.
In particular, once a robot reaches its target it stays there.

\Cref{fig:example_instance} shows an example of an instance.
Note that MRMP-GP is a feasibility problem and not an optimization problem; hence, we only move one robot at a time in a solution.

\paragraph{Variants.} Now we define the two variants of the problem we consider.
Let $r$ and $r'$ be two different robots.
We say that $\trg{r'}$ is a \textit{blocking target} of $r$ if $\trg{r'}$ lies on $\pi(r)$.
An instance is said to have the \textit{non-blocking targets} property if it does not have blocking targets.
This case corresponds to the L-SU-RAS resource allocation problem~\cite{DBLP:journals/automatica/Reveliotis20} because in this problem when a process completes it is essentially gone from the world. This is equivalent to a robot whose target is not a blocking target of any other robot since we can consider the robot as disappearing once it reaches its target.

For our second variant, an edge $(u,v)$ in $W$ is called \emph{bi-directional} if it has to be traversed in both directions by different robots, i.e., there are two robots $r$ and $r'$ where $(u,v)$ appears on $\pi(r)$ and $(v,u)$ appears on $\pi(r')$.
An instance is said to have \textit{one-way} or \textit{unidirectional} motion if it has no bi-directional edges.

\paragraph{Parameters.}
For a vertex $v \in V(W)$, we denote by $N(v)$ the number of paths in which it appears, i.e.,  $N(v) \coloneqq |\set{r \mid v \in \pi(r), r \in R}|$.
We define \emph{vertex multiplicity} of an MRMP-GP instance $M$, denoted by $\VM(M)$, to be $\max_{v \in V(W)} N(v)$.
For the case where $W$ is a 2D grid, we examine path shape complexity. We define the \emph{turn number} to be the maximum number of turns, i.e., the minimum number of line segments needed to draw a path on the grid minus one, made by any input path.

\begin{figure}[t]
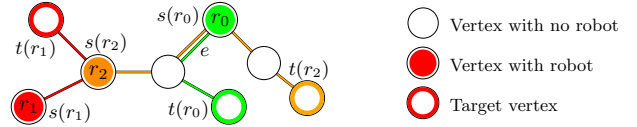
 \centering
    \begin{subfigure}{0.2\textwidth} \centering
        \includegraphics[page=1, scale=0.82]{figures/intro.pdf}
    \end{subfigure}
    \hspace*{\fill}
    \begin{subfigure}{0.2\textwidth} \centering
        \includegraphics[page=2, scale=0.75]{figures/intro.pdf}
    \end{subfigure}
    \caption{An MRMP-GP instance in which $W$ has $8$ vertices. There are $3$ robots, each with its own source, target and path which are shown in the robot's color.
    The instance has the non-blocking targets property but does not have one-way motion (due to the edge $e$). Its vertex multiplicity is 2.
    A possible solution to this instance is to move each robot all the way to its target in the order $r_0$, $r_2$, $r_1$.
    (While it is not needed in this example, in general we allow a robot to stay put in an intermediate vertex along its path, while other robots move.)
} %
\label{fig:example_instance}
\end{figure}

\begin{figure}[t]
  \includegraphics[page=1, width=0.28\textwidth]{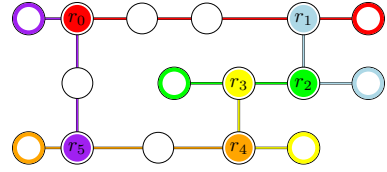}
 \caption{ A simple cycle of $6$ robots, $r_0,\ldots,r_5$, where each robot is blocking the previous one from reaching its target.}

\label{fig:simple_cycle}
\end{figure}

\textbf{Shorthand notation.}
We use the following shorthand notation throughout the paper.
For the problem variant with non-blocking targets, we use MRMP-GP-NBT($x$), where $x$ indicates the vertex multiplicity.
We similarly use MRMP-GP-UNI($x$) for the variant that has uni-directional motion.
When $W$ is a 2D grid, we also indicate the turn number as the second parameter, e.g., MRMP-GP-UNI(2,0) indicates a VM of 2 and straight paths (i.e., 0 turns).

\section{Algorithm}
\label{sec:main-alg}
We now consider MRMP-GP-NBT(2), i.e., the variant of MRMP-GP having the non-blocking targets property with a vertex multiplicity of at most 2.
We present a polynomial-time algorithm that finds a solution or reports that none exists.
The algorithm has two phases, which we now describe.

\noindent \paragraph{Phase~1.} A robot $r$ located at vertex $v \in \pi(r)$ is said to have a \emph{clear path} if there are no other robots along the remainder of its path, i.e., the subpath of $\pi(r)$ from $v$ to $t(r)$.
Given the restriction that a robot's target is not included in any other robot's path, a robot with a clear path that is moved to its target will not block any other robot in the future.
Hence, we move each robot with a clear path to its target until no such robot exists.
Note that moving one robot may clear the path of another robot, therefore the robots are checked repeatedly.

\noindent \paragraph{Phase~2.} Next, the algorithm iteratively identifies and solves \emph{cycles}.
In each iteration, a single cycle is solved, which amounts to moving only the cycle's robots. 
At the beginning of each iteration, we have the following invariant:
\begin{quote}
Each robot either
(i) reached its target vertex, or 
(ii) is still at its source vertex and it does not have a clear path. 
\end{quote}

At the end of Phase~1, the invariant is clearly maintained.
A robot $r$ is said to be \emph{blocked} by another robot $r'$ if the path from the current position of $r$ to $t(r)$ contains the current position of $r'$.
At the beginning of each iteration, any robot that has not reached its target yet is blocked by some other robot.
Let $r_0$ be one such robot and let $r_1$ be the robot blocking it, which must therefore not be at its target $t(r_1)$.
Then robot $r_1$ has not yet reached its target, therefore the path $\pi(r_1)$ contains some other robot $r_2$, which blocks $r_1$, and so on.
The blocking relationship yields a sequence $r_0,r_1,r_2,\ldots,r_{h-1}$, where $r_{h-1}$ is the first robot that is blocked by a robot $r_i$ already appearing in the sequence, i.e., $i<h-1$.
See Figure~\ref{fig:simple_cycle} for an illustration.
To simplify notation, from this point on we use indices of robots modulo $h$, i.e., we write $r_i$ instead of $r_{(i \mod h)}$.

\begin{lemma}
A blocking sequence $r_0,r_1,...,r_{h-1}$ that ends when $r_{h-1}$ is blocked by some robot $r_i$, for $0 \leq i < h-1$, forms a \textit{cycle}, i.e., each robot $r_i$ is blocked by $r_{i+1}$.
\end{lemma}

\begin{proof}
Any robot $r_j,j<h-1$ is blocked by $r_{j+1}$ by definition. Hence, it remains to prove that $r_{h-1}$ is blocked by $r_0$.
Assume for a contradiction that $r_{h-1}$ is blocked by a robot $r_i$ with $i>0$.
This means that the paths of the robots $r_{i-1},r_i,r_{h-1}$ all contain the start vertex $\src{r_i}$, contradicting the assumption that each vertex is contained in at most two robots' paths.
\end{proof}

Therefore, all the robots that are not at their target can be divided into disjoint cycles.
The algorithm \textit{solves} each cycle (in a sense that we define next) independently, and then moves the cycle robots to their targets. This is possible since, as stated in \Cref{lemma:solved_cycle_clear_path} below, their paths must be clear.

Let us fix a cycle $C = r_0,r_1,...,r_{h-1}$.
Each robot $r_i \in C$ is currently at its source vertex.
\textit{Solving} the cycle $C$ is defined as moving each robot $r_i$ to the start position of the next robot, i.e., $s(r_{i+1})$, and $s(r_{i+1})$ is called the \textit{cycle target} of $r_i$.

\begin{lemma} \label{lemma:solved_cycle_clear_path}
After $C$ is solved, all the robots in $C$ have a clear path.
\end{lemma}
\begin{proof}
    Assume by contradiction that some robot $r_i$ of $C$ is blocked by another robot $r$ after solving $C$.
    If $r$ is not in $C$, then $r$ belongs to another cycle, where it blocks another robot $r'$.
    This means that $r$ blocks two robots at its current vertex, which contradicts the assumption that vertex multiplicity is 2.
    If $r$ is also part of $C$, then $r$ was blocked by a robot $r'$ in $C$.
    After solving $C$, $r$ is at the source vertex $\src{r'}$, which implies that $\src{r'}$ appears on 3 robot's paths, namely, $r$, $r'$, and $r_i$, which again contradicts the assumption of a vertex multiplicity of 2.
\end{proof}

\paragraph{Solving a cycle.} %
The subpath of robot $r_i$ from $\src{r_i}$ to $\src{r_{i+1}}$ is called the \textit{cycle path} of robot $r_i$.
If there is a robot $r_j$ with a vertex $p$ on its cycle path that does not appear in the cycle path of another robot in $C$, then we call $r_j$ a \textit{scout}. For example, $r_0$ in \Cref{fig:simple_cycle} is a scout robot.
If $C$ contains a scout robot, then $C$ can be solved by moving the scout robot to $p$ and then moving each robot in the reverse order from the scout robot along its cycle path.
See \Cref{sec:scout-cycle-algo} for full details of this simple case.

\begin{figure}[h] \centering
    \includegraphics[page=4, width=0.47\textwidth]{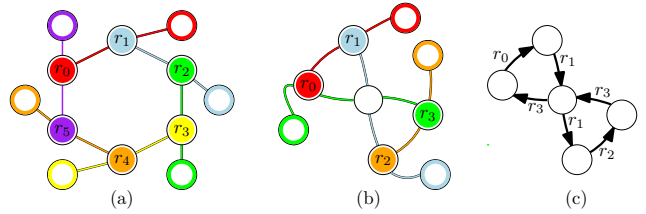}
    \caption{Unsolvable instances, each containing a single cycle: (a) A deadlock where no robot can move. (b) Either $r_1$ or $r_3$ can move, but either move will result in a deadlock of three robots similar to (a). (c) The graph $G$ constructed by the algorithm for the cycle in (b).}
    \label{fig:unsolvable_cycles} %
\end{figure}

If there is no scout robot in $C$, a solution might not exist; see \Cref{fig:unsolvable_cycles}.
To determine whether a cycle is solvable, we construct a directed graph~$G$ with vertices $V(G)\subset V(W)$ and edges $E(G)$.
The vertices $V(G)$ are a subset of the vertices on the cycle paths of the robots in $C$.
For each robot $r$ in $C$ and each directed edge $e$ on the cycle path of $r$, we add the edge $e$ to $E(G)$.
We label this edge by $r$ to remember which robot induced it.
It is therefore possible that we have two edges from one vertex to another in $G$, but they will be labeled by different robots.
We say that a directed graph with labeled edges of this form is \emph{composed} of paths; see \Cref{fig:composedPaths}. %
Note that the graph $G$ has the following properties:
\begin{itemize} [topsep=2pt]
    \setlength\itemsep{0pt}
    \item For each distinct label, the edges with that label form a path.
    \item Each vertex appears on the paths of exactly two labels.
    \item The end vertex of a path with one label is the start vertex of a path with another label.
    \item The start vertex of a path with one label is the end vertex of a path with another label.
    \item There is an Eulerian  cycle in $G$ such that for each label, the edges with that label are traversed consecutively in the cycle.
\end{itemize}
Any directed graph with labeled edges that has these properties is said to be \emph{composed} of paths.
A \emph{start} vertex of such a graph $G$ is the start vertex of the path with any label, and we denote the set of start vertices as $V_s(G)$.

\begin{figure}
  \centering
  \includegraphics[page=9, width=0.46\textwidth]{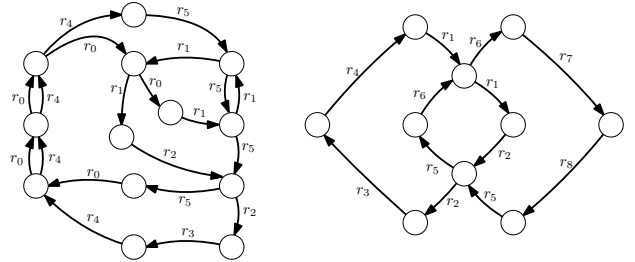}
  \caption{Left: A graph composed of paths, which is solvable according to \Cref{lemma:basecases-solvable}.
  Right: A graph that is not composed of paths, since the edges do not form an Eulerian cycle of the required type.}
  \label{fig:composedPaths}
\end{figure}

A graph $G$ which is composed of paths represents an equivalent robot cycle obtained in the following simple way.
For the path $\pi$ in $G$ consisting of all edges with label $r$, we place the robot $r$ on the start vertex of $\pi$, and the robot $r$ must traverse the path $\pi$.
See \Cref{fig:unsolvable_cycles}(c) for illustration.
If this instance of MRMP-GP has a motion plan, we say that $G$ is \emph{solvable}.

Let us first prove an elementary lemma about the degrees of a graph $G$ composed of paths.
Let $\din(v)$ and $\dout(v)$ denote the in- and out-degree of a vertex $v$ of $G$.

\begin{lemma} \label{lemma:degree11} \label{lemma:degree22}
Let $G$ be a graph composed of paths and $v$ be a vertex of $G$.
If $v \in V_s(G)$ then we have $(\din(v),\dout(v))=(1,1)$ and otherwise $(\din(v),\dout(v)) = (2,2)$.
\end{lemma}

\begin{proof}
Consider $v \in V_s(G)$.
By the properties of $G$, the path of one label starts at $v$, and the path of another label ends at $v$.
Since $v$ only appears on the paths of two labels, the statement follows.

Now consider a vertex $v \in V(G)\setminus V_s(G)$.
By the properties of $G$, the vertex~$v$ is included in the paths of two labels and is not a target vertex of either of them.
This means that on both paths, there are edges to $v$ and edges out of $v$, so $(\din(v),\dout(v)) = (2,2)$.
\end{proof}

We now turn our attention to classifying the solvable graphs that are composed of paths.
\Cref{lemma:basecases-unsolvable,lemma:basecases-solvable} present the \textbf{base cases} of solvable vs.\ unsolvable graphs:

\begin{lemma} \label{lemma:basecases-solvable}
Consider a graph $G$ that is composed of paths.
If all simple cycles in $G$ have at least two vertices that are not start vertices, then $G$ is solvable.
\end{lemma}

\begin{restatable}{lemma}{unsolvablebase} 
\label{lemma:basecases-unsolvable}
If a graph $G$ composed of paths contains a simple cycle where all vertices are start vertices, then $G$ is not solvable.
\end{restatable}

The proof of \Cref{lemma:basecases-unsolvable} is given in \Cref{sec:app:base}.
We postpone the proof of \Cref{lemma:basecases-solvable},  to first describe the algorithm that solves the corresponding instance of MRMP-GP; see Figure~\ref{fig:no_loops} for a non-trivial example.

The algorithm begins by partitioning the robots into blocks as follows.
Consider a maximal directed path $\pi$ in $G$ with at least $3$ vertices such that each vertex of $\pi$ except the first and last is a start vertex.
Let $r_1,\ldots,r_k, k \geq 1$ be the corresponding robots starting on these start vertices in order, and let $B_1=(r_1,\ldots,r_k)$.
Running over all such paths $\pi$, we obtain blocks of robots $B_1,B_2,\ldots$.

For a block $B_i=(r_1,\ldots,r_k)$, we then define $\head(B_i)=r_k$ and $\tail(B_i)=r_1$.
The robot $r_k$ has its \textit{cycle target} at the start vertex of a robot $r'_1$ in another block $B_j$.
We then define $\nextb(B_i)=B_j$ and $\prevb(B_j)=B_i$.
Due to the properties of a graph composed of paths, these relations are defined for all blocks $B_i$.

We now describe our algorithm; See Algorithm~\ref{alg:no_cycles} for pseudo-code.
We first choose an arbitrary block $B_i$.
Then, we move all robots of $B_i$ a single edge forward from $\head(B_i)$ to $\tail(B_i)$.
Move $\head(B_i)$ along its cycle path until it cannot be moved any further.
Let $B_j=\nextb(B_i)$ be the block containing the robot, $\tail(B_j)$, that blocks $\head(B_i)$.
Move the robots $B_j$ forward by one edge and then move $\head(B_i)$ to its cycle target, which is the original position of $\tail(B_j)$.
Move $\head(B_j)$ as much as possible.
From here on, the pattern repeats until all the blocks have been traversed.

\begin{algorithm}
    \caption{%
    Solve an instance of MRMP-GP corresponding to a graph $G$ composed of paths satisfying the requirements of \Cref{lemma:basecases-solvable}.}
    \label{alg:no_cycles}
    $B_{\textit{start}}\gets B_1$\;
    $B \gets B_{\textit{start}}$\;
    \Loop{} { \label{alg_no_cycle:loop}
        \For{$r \in B$
        in order from $\head(B)$ to $\tail(B)$} {
            Advance $r$ one edge\; %
        }
        \If{$B$ is not $B_{\textit{start}}$} {
            Advance $\head(\prevb(B))$ one edge\label{alg2:line7}\; %
        }
        Advance $\head(B)$ as much as possible\label{alg2:line8}\;
        $B \gets \nextb(B)$\;
        \If{$B$ is $B_{\textit{start}}$} {
            \Return\;
        }
    }
\end{algorithm}

Note that this algorithm can move the head of a block three times.
It might be tempting to suggest that line~\ref{alg2:line8} of the algorithm can be avoided (by instead advancing $\head(\prevb(B))$ as much as possible in line~\ref{alg2:line7}), but this is not the case, as shown in \Cref{fig:threemoves}.

\begin{figure}
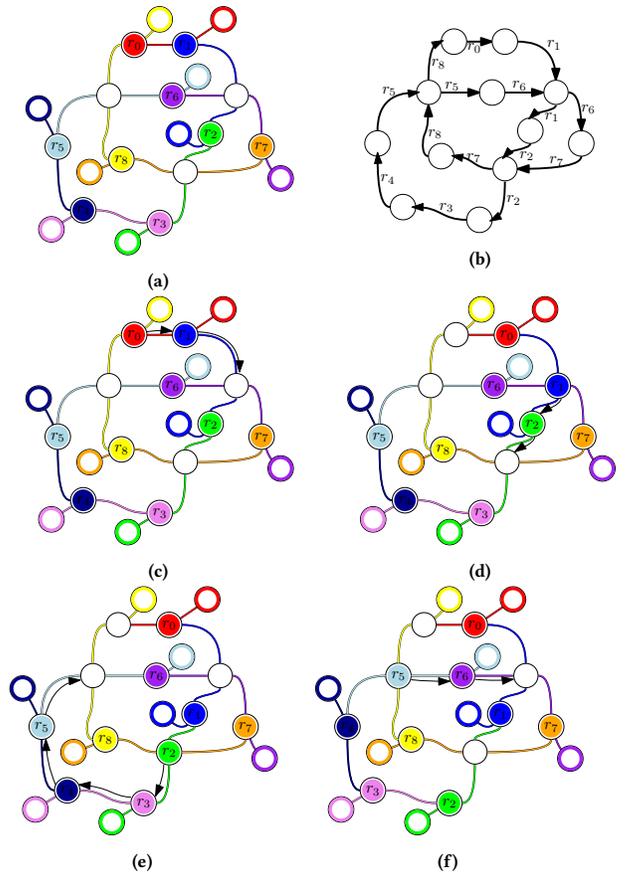
 \centering
    \hspace*{\fill}
    \begin{subfigure}{0.2\textwidth} \centering
        \includegraphics[page=1, scale=0.6]{figures/no_loops.pdf}
        \caption{} \label{}
    \end{subfigure}
    \hspace*{\fill}
    \begin{subfigure}{0.2\textwidth} \centering
        \includegraphics[page=9, scale=0.6]{figures/no_loops.pdf}
        \caption{} \label{}
    \end{subfigure}
    \hspace*{\fill}
    \begin{subfigure}{0.2\textwidth} \centering
        \includegraphics[page=2, scale=0.6]{figures/no_loops.pdf}
        \caption{} \label{}
    \end{subfigure}
    \hspace*{\fill}
    \begin{subfigure}{0.2\textwidth} \centering
        \includegraphics[page=3, scale=0.6]{figures/no_loops.pdf}
        \caption{} \label{}
    \end{subfigure}
    \hspace*{\fill}
    \begin{subfigure}{0.2\textwidth} \centering
        \includegraphics[page=4, scale=0.6]{figures/no_loops.pdf}
        \caption{} \label{}
    \end{subfigure}
    \hspace*{\fill}
    \begin{subfigure}{0.2\textwidth} \centering
        \includegraphics[page=5, scale=0.6]{figures/no_loops.pdf}
        \caption{} \label{}
    \end{subfigure}
    \hspace*{\fill}

    \caption{
    (a) An instance with a single blocking cycle.
    (b) The graph $G$ created for the single cycle, in which all simple cycles contain two free (non-start) vertices.
    The blocks of maximal consecutive robots are $B_1=(r_0,r_1), B_2=(r_2), B_3=(r_3,r_4,r_5), B_4=(r_6), B_5=(r_7), B_6=(r_8)$.
    We illustrate a few steps of the solution obtained by
    Algorithm~\ref{alg:no_cycles}:
    (c) $B_1=(r_0,r_1)$ is the first chosen block, for which robots are moved head to tail.
    (d) The robots of $\nextb(B_1)$, $B_2=(r_2)$, are moved head to tail, and then $\head(B_1)$ is moved again.
    (e) The robots $\nextb(B_2)$, $B_3=(r_3,r_4,r_5)$, are moved head to tail, and then $\head(B_2)$ is moved.
    (f) $B_4=(r_6)$ and $\head(B_3)$ are moved.
    }
    \label{fig:no_loops}
\end{figure}

\begin{figure}[h]
\centering
\includegraphics[page=8, width=0.4\textwidth]{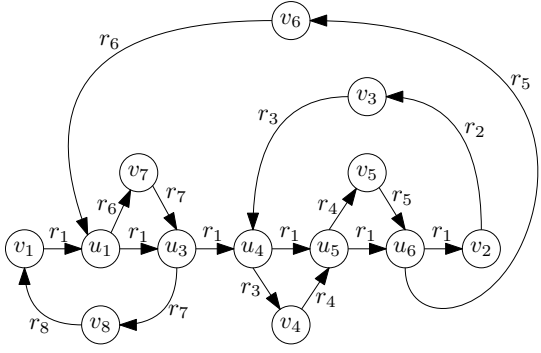}
\caption{Suppose that we first move the robot $r_7$ from $v_7$ to $u_3$.
If our next move is to advance $r_1$ and $r_8$, as dictated by Algorithm~\ref{alg:no_cycles}, then it is unavoidable to move $r_1$ three times.
Indeed, after clearing $u_3$, $r_1$ must move to $u_6$ (or $u_5$) and wait for $v_2$ to be cleared.
Otherwise, if $r_1$ stays at $u_1$, then one of the vertices $u_4,u_5,u_6$ becomes occupied when freeing $v_2$, so $r_1$ cannot move to $v_2$ in a single move. }
\label{fig:threemoves}
\end{figure}

\begin{proof}[Proof of \Cref{lemma:basecases-solvable}.]
We show that Algorithm~\ref{alg:no_cycles} solves $I(C)$, the instance of MRMP-GP corresponding to $G$ of cycle $C$.
When solving the cycle, we move each block of robots $B_i$ from head to tail.
Clearly, only moving $\head(B_i)$ can be an invalid motion, as the rest of the robots in $B_i$ are always moved to a vertex that was just evacuated by the previously moved robot.

\begin{figure*}
\includegraphics[page=3, width=0.98\textwidth]{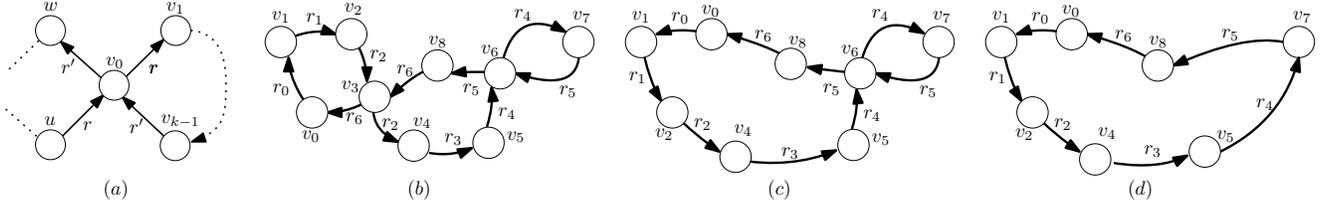}
\caption{
(a) A cycle $v_0,\ldots,v_{k-1}$ that requires untangling.
(b) A graph composed of paths that requires two untangling operations, shown in (c) and (d), after which we conclude that the graph is unsolvable.}
\label{fig:untangling}
\end{figure*}

Denote by $u,v$ the vertices that are occupied by the tail and head of $B_1$ respectively ($u,v$ might be the same vertex if $B_i$ consists of a single robot).
Both $u,v$ are in $V_s(G)$, so by \Cref{lemma:degree11}, $u$ has a single in-going edge, from another vertex $w$, and $v$ has a single out-going edge, to another vertex $z$.
Both $w,z$ are not occupied by robots and are not in $V_s(G)$, otherwise these robots would have been included in $B_i$.
If $w=z$, then $w$ and all the vertices occupied by the robots of $B_i$ form a simple cycle with a single non-start vertex, in contradiction to the lemma's assumption, therefore $w \neq z$.

At the beginning of an iteration of the main loop (line~\ref{alg_no_cycle:loop}), in which $B_i$ is handled, we have the following invariant:
At most one vertex of $V(G) \setminus V_s(G)$ is occupied by a robot, and the occupied vertex is the source of the edge directed to the vertex $\src{\tail(B_i)}$.
In other words, the occupied vertex is $w$, as defined above.
We now verify that the invariant is maintained.

Initially, all the robots are on $V_s(G)$ and the invariant is maintained.
Suppose now that the invariant holds at the beginning of an iteration.
Before moving any robot of $B=B_i$, $w$ is the only non-start vertex that can be occupied, therefore $z$ is never occupied, and moving $\head(B_i)$ a single step to $z$ is valid since $w\neq z$.
After moving the head of $B_i$, the rest of the robots of $B_i$ are moved head to tail, and the robot that occupied $w$ is moved to the original position of the tail of $B_i$.
Lastly, the head of $B_i$ is moved as much as possible, until it reaches $w'$, the source vertex of the in-edge of $\src{\tail(\nextb(B_i))}$, and the invariant is maintained.
When the last block is handled, the head is moved as much as possible and reaches $\src{\tail(\nextb(B_i))}$.
Therefore the motion is valid.

We now verify that all robots get to their cycle targets.
After the algorithm handles a single block $B_i$, all robots in $B_i$ except $\head(B_i)$ reach their cycle targets.
As part of handling the next block $\nextb(B_i)$, $\head(B_i)$ is moved again and reaches its cycle target, so all the robots in $B_i$ traverse their cycle path.
If $B_i$ is the last block handled, then all the robots of other blocks reached their cycle targets, including $tail(B_{\textit{start}})$, therefore when $head(B_i)$ is advanced as much as possible it will reach its cycle target, the start position of $tail(B_{\textit{start}})$.
All the blocks $B_1,B_2,\ldots$ are examined, therefore all the robots reached their cycle targets and the cycle is solved. %
\end{proof}

If $G$ does not fall under one of the base cases of \Cref{lemma:basecases-unsolvable,lemma:basecases-solvable}, we gradually reduce $G$ to a graph that does fall under one of these cases, as we describe next.
So suppose that we are in a situation not included in the base cases, namely that $G$ has a simple cycle $v_0,\ldots,v_{k-1}$ where $v_0$ is not a start vertex, but all other vertices are.
Note that we must have $(\din(v_0),\dout(v_0))=(2,2)$, according to \Cref{lemma:degree22}.
It follows that $(v_0,v_1), (v_{k-1},v_0), (u,v_0), (v_0,w)$ are edges in $E(G)$ incident to $v_0$, where $u,w$ are vertices not appearing in the simple cycle $v_0,\ldots,v_{k-1}$.
Moreover, there are two robots $r$ and $r'$ such that $(u,v_0)$ and $(v_0,v_1)$ are labeled with $r$, and $(v_{k-1},v_0)$ and $(v_0,w)$ are labeled with $r'$, as illustrated in \Cref{fig:untangling}(a).
If all four edges had the same label $r$ then $r$ would be a scout robot.
The edges $(v_{k-1},v_0)$ and $(v_0,v_1)$ cannot have the same label, as then the robots on $v_0,\ldots,v_{k-1}$ would form a blocking cycle independent of the rest of the blocking cycle $C$ currently being solved.

We now define an \emph{untangling} operation which leads to a new graph $G'$ with vertices $V(G')=V(G)\setminus\{v_0\}$; see \Cref{fig:untangling}.
In $G'$ we replace the two edges labeled by $r$ with a single edge $(u,v_1)$ and keep the label $r$.
We likewise replace the two edges labeled $r'$ with a single edge $(v_{k-1},w)$ labeled $r'$.

We now prove that the untangling operation maintains solvability.

\begin{lemma}\label{lemma:maintainComposed}
\label{lemma:untangling}
Let $G$ be a graph composed of paths and let $G'$ be the graph obtained by performing the untangling process in $G$.
Then (i) $G'$ is also a graph composed of paths and (ii) $G'$ is solvable if and only if $G$ is solvable.
\end{lemma}

\begin{proof}
Consider the untangling operation described above and the involved vertices $v_{k-1},v_0,v_1,u,w$.
First, it is easy to verify that the operation maintains the properties of a graph composed of paths.
Now suppose that $G$ is solvable.
Two robots should pass through $v_0$: denote by $r_u$ the robot that includes $u,v_0,v_1$ in its cycle path and $r_w$ the robot that includes $v_{k-1},v_0,w$ in its cycle path.
There is no robot at $v_0$, and the algorithm should decide which of $r_u,r_w$ will go through $v_0$ first.
If $r_u$ is moved to $v_0$ before $r_w$, the robots enter a deadlock, as the vertices $v_0,\ldots,v_{k-1}$ will all contain robots (i.e., the same situation as in~\Cref{lemma:basecases-unsolvable}).
Therefore $r_w$ must move to $v_0$ before $r_u$.
So $r_w$ moves first to $v_0$ and then eventually to $w$, and then $r_u$ moves to $v_0$ and then to $v_1$.
It follows that in $G’$, there is a motion plan where $r_w$ first moves on the merged edge $(v_{k-1},w)$, then all robots on the path $v_1,\ldots,v_{k-2}$ move, and then $r_u$ moves on the merged edge $(u,v_1)$.

Suppose now that $G’$ is solvable.
We can then simulate the same solution in $G$:
When $r_w$ traverses the edge $(r_{k-1},w)$ in $G'$, we let it traverse both edges $(r_{k-1},v_0)$ and $(v_0,v_w)$ in $G$.
Likewise, when $r_u$ traverses $(u,v_1)$ in $G'$, we let it traverse both edges $(u,v_0)$ and $(v_0,v_1)$ in $G$.
All other edges in $G$ and $G'$ are the same, and for these we copy the moves directly.
We then have a solution for $G$.
\end{proof}

Our algorithm proceeds by untangling cycles until it is no longer possible.
This results in a sequence of graphs $G_0,G_1,\ldots,G_{m}$, where $G_0=G$ is the original graph and each $G_i$ results from performing the untangling operation in $G_{i-1}$.
By \Cref{lemma:maintainComposed}, all these graphs are composed of paths, and the resulting graph $G_{m}$ is solvable if and only if the original graph $G$ is solvable.
Furthermore, there is nothing more to untangle in $G_{m}$, so we either have a simple cycle with robots on all vertices or there are at least two vertices with no robots on all simple cycles.
It then follows from \Cref{lemma:basecases-unsolvable,lemma:basecases-solvable} that there is a solution if and only if we are in the latter case.

After solving a cycle, all the cycle robots have a clear path to their target  (see~\Cref{lemma:solved_cycle_clear_path}), and the algorithm moves all of them to their targets one by one.
A single iteration of Phase~2 has ended, and the invariant that each robot is either at its start position or at its target position is maintained, as the only robots the algorithm moved were the cycle robots, and all of them reached their targets.

At any step of the algorithm, we move robots only when we have a guarantee that there is a solution after the move if and only if there is a solution before it.
Therefore, if the algorithm does not find a valid solution, namely one of the cycles was untangled to a graph with a simple cycle with robots on all vertices, there is no solution.
Otherwise, the algorithm finds a valid motion plan for all robots.

\paragraph{Running Time Analysis.}
Our algorithm can be implemented in time linear in the sum of the lengths of the robots' given paths, which we denote by $L$. %
We now provide the details on how to obtain such a running time.

To implement Phase~1 efficiently, first, for each robot $r$ we identify the first robot that blocks it. Then each robot not blocked by other robots is moved to its target.
Whenever a robot $r$ is moved to its target, if $r$ blocked robot $r'$ beforehand, we check whether $r'$ is blocked by some other robot subsequently.
Only the suffix of $\pi(r')$ can be checked, starting from the vertex $\src{r}$.
Hence, each vertex in the given paths is examined exactly the number of times it appears in the given paths during Phase~1.

To implement Phase~2 efficiently, we first divide the robots into independent cycles by checking for each robot which robot blocks it, which requires $O(L)$-time.
Then, each cycle $C$ is solved independently. If it contains a scout robot, clearly all computations are linear in the lengths of the robots' paths.
If $C$ does not contain a scout robot, we construct a graph composed of paths $G$ whose size is linear in the length of the robots paths.
If $G$ falls under one of the base cases of \Cref{lemma:basecases-unsolvable,lemma:basecases-solvable}, we can easily obtain a solution or declare these is no such one in time linear to the graph size.

To perform all the untangling operations efficiently in case $G$ does not fall under one of these base cases, firstly the blocks of maximal consecutive robots $B_1,B_2,\ldots$ are computed.
All vertices occupied by the robots of any $B_i$ have single in- and out-edges by \Cref{lemma:degree11}, and any vertex occupied by a robot is included in some $B_i$.
Therefore to identify a simple cycle with only a single vertex that is not occupied by a robot it suffices to check whether there is a $B_i$ such that the source vertex of the in-edge of $\tail(B_i)$ is the same as the target vertex of the out-edge of $\head(B_i)$.
When a block satisfies this property, we say that it is an \textit{knot} block.
We proceed by iterating over the blocks using the $\nextb(\cdot)$ relationship, starting from an arbitrary block $B$:
If $B$ is not a knot block, we continue to the next block, i.e., we set $B \gets \nextb(B)$.
If $B$ is a knot block, we perform the untangle operation on the in-edge of $\tail(B)$, the out-edge of $\head(B)$ and the shared vertex $u$ of both, resulting in the removal of $u$. %
A union of $B$ and $\prevb(B),\nextb(B)$ may be required, if $u$ was the only vertex between them, and in such scenario the next block to be checked is the resulting union rather than $\nextb(B)$.
We stop after iterating over all the blocks once.
At each iteration, we either advance using $\nextb(\cdot)$, or unite at least $2$ blocks, bounding the number of iterations by $2|\{B_1,B_2,\ldots\}|$ of $G_0$ (the original graph before untangling).
Each non-knot block will remain that way.
New knot blocks may be created only during the union of $\prevb(B),B,\nextb(B)$, which are created during an untangling operation, and the new blocks are checked in the next iteration.
Therefore there will not be more knot blocks after a single full traversal of the blocks $B_1,B_2,\ldots$.

After untangling a cycle graph $G$, it must fall under one of the base cases, which determine whether it is solvable.
If it is, we compute the robots' moves using Algorithm~\ref{alg:no_cycles}, which clearly runs in time linear in $|V(G)|$.
Hence, one iteration of Phase~2 solving a single cycle $C$ requires time linear in the sum of path lengths of the robots in $C$.
Since each cycle is solved once and each robot appears in at most one cycle, Phase~2 takes $O(L)$-time. %

Overall, both phases of the algorithm require linear time in the total lengths of the given paths.
Putting everything together, we obtain the following:

\begin{theorem}
MRMP-GP-NBT(2) has an algorithm that finds a solution or reports that no solution exists, in time linear in the total lengths of the given paths.
\end{theorem}

\input{hardness_full}

\section{Conclusion}
We gave a refined complexity analysis of MRMP-GP that sheds new light on the problem's sources of difficulty.
A key element of previous MRMP-GP hardness constructions is paths that traverse the same set of vertices in opposite directions (e.g., a given path would contain the sequence $v_1,v_2,v_3$ while another path would contain $v_3,v_2,v_1$).
We show that hardness can arise even without such paths if instead we have a different element, which is blocking targets.
This observation leads to an intriguing question, which is whether the MRMP-GP remains hard when neither elements are present.
A positive answer could have implications for fixed-path robot/transport systems, %
which could be designed to avoid the aforementioned elements.

From the perspective of parameterized complexity~\cite{cygan2015parameterized}, which is a research avenue for hard motion planning problems~\cite{DBLP:conf/atal/SalzmanS20},
our hardness results rule out candidate parameters.
Namely, by showing that hardness remains even for a constant vertex multiplicity (VM) we prove that MRMP-GP is unlikely to be fixed-parameter tractable (FPT) when parameterized by VM.
The same statement holds for path shape complexity.
Therefore, we guide the search for parameterized algorithms toward other parameters.
We believe that our hardness constructions more vividly expose potential parameters since our results hold for highly distilled MRMP-GP formulations.

A natural extension of MRMP-GP is optimizing the solution, e.g., its makespan.
Such an optimization variant is closely related to Multi-Agent Path Finding~\cite{DBLP:conf/socs/SternSFK0WLA0KB19}, where its use as a subroutine may have potential.
For example, in each high-level node of the popular Conflict-Based Search algorithm~\cite{DBLP:journals/ai/SharonSFS15} a path is found for each agent, which may be viewed as fixed for that node.
Hence, a fast algorithm for the optimization variant of MRMP-GP might improve lower bounds for the cost of a high-level node, thus better guiding the search.
We believe that our new insights for deciding feasibility provide a better foundation for such future directions.

\begin{acks}
M.~Abrahamsen is supported by Starting Grant 1054-00032B from the Independent Research Fund Denmark under the Sapere Aude research career program and is part of Basic Algorithms Research Copenhagen (BARC), supported by the VILLUM Foundation grant 16582.
Work on this paper by T.~Geft and D.~Halperin has been supported in part by the Israel Science Foundation (grant no.~1736/19), by NSF/US-Israel-BSF (grant no.~2019754), by the Israel Ministry of Science and Technology (grant no.~103129), by the Blavatnik Computer Science Research Fund, and by the Yandex Machine Learning Initiative for Machine Learning at Tel Aviv University.
T.~Geft has also been supported by scholarships from the Shlomo Shmeltzer Institute for Smart Transportation at Tel Aviv University and the Israeli Smart Transportation Research Center.
\end{acks}

\bibliography{references.bib}
\balance
\bibliographystyle{ACM-Reference-Format}

\pagebreak
\appendix

\section{Missing details of our algorithm}
\label{sec:appendix}

\subsection{Solving a cycle with a scout robot} \label{sec:scout-cycle-algo}
Let $r_j$ be a scout robot in a blocking cycle $C$.
That is, there is a vertex $p$ on $r_j$'s cycle path that is not contained in the cycle path of any other robot in $C$.
We solve $C$ by moving $r_j$ to $p$, and then moving each robot in $r_{j-1},r_{j-2},\dots,r_0, r_{h-1}, r_{h-2} \ldots,r_{j}$ to the next robot's start vertex (in this order).
See Algorithm~\ref{alg:scout-cycle} for pseudo-code and \Cref{fig:cycle-with-scout} for an example.

\begin{algorithm} 
    \caption{Solving a cycle with a scout robot $r_j$ }
    \label{alg:with_gap}
    \label{alg:scout-cycle}
    Move $r_j$ to $p$\;
    \For{$i=j-1,\ j-2,\ \ldots,\ 0,\ h-1,\ h-2,\ \ldots,\ j$} {
        Move $r_i$ to $\src{r_{i+1}}$\;
    }
\end{algorithm}

\begin{figure}[h]
  \centering
  \includegraphics[page=1, width=0.46\textwidth]{figures/scout_robot.pdf}
  \caption{A blocking cycle with a scout robot $r_3$, which is solved as described in the main text.
  The initial state is on the top left configuration and the bottom right state shows the solved cycle.
  First $r_3$ is moved to a vertex that does not appear on any other robot's path.
  Then all other robots are moved sequentially.}
  \label{fig:cycle-with-scout}
\end{figure}

\begin{lemma}
Let $C$ be a cycle containing a scout robot $r_j$, i.e., $r_j$ has vertex $p$ on its cycle path that is not contained in the cycle path of any another robot in $C$. Then Algorithm~\ref{alg:scout-cycle} solves $C$.
\end{lemma}

\begin{proof}
Before moving any robot, each robot $r_i$ has a single other robot on its cycle path, and that is $r_{i+1}$, by the definition of the cycle.
The motion of the scout robot $r_j$ to $p$ is valid as the only robot in $r_j$'s cycle path is $r_{j+1}$ which is at the last vertex on the path, after $p$.
Before any robot $r_i$ is moved, any robot $r_k,i+1 \leq k \leq j-1$ was moved to $\src{r_{k+1}}$, and due to vertex multiplicity of 2, the vertex $\src{r_{k+1}}$ is included in at most two robots paths, $r_k,r_{k+1}$, and it is not included in the path of $r_i$, therefore $r_k$ will not block $r_i$.
The scout robot will not block $r_i$ as $p$ is only included in the cycle path of $r_j$.
Any robot $r_k,j+1 \leq k \leq i-1$ is at its original position, blocking only $r_{k-1}$ and not $r_i$.
Therefore the motion of $r_i$ to $\src{r_{i+1}}$ is not blocked by any robot.
After iterating over all the robots, the algorithm moves each robot $r_i$ to $\src{r_{i+1}}$ and by definition solves the cycle.
\end{proof}

\subsection{Proof of the unsolvable base case} \label{sec:app:base}
\unsolvablebase*
\begin{proof}
Consider such a cycle $C=v_0,\ldots,v_{k-1}$.
In the corresponding instance of MRMP-GP, $I(C)$, we have a robot $r_i$ on each vertex $v_i$.
We claim that for all $i$, the path for $r_i$ starts with an edge from $v_i$ to $v_{i+1}$.
Otherwise, using the properties of graphs that are composed of paths, it is easy to conclude that the vertex $v_i$ appears on at least three different paths, which is not possible.
We therefore have that $r_i$ cannot move, as it is only allowed to move to $v_{i+1}$ according to its path, which is occupied by $r_{i+1}$.
In other words, the cycle forms a deadlock and there is no solution.
\end{proof}

\subsection{Efficient implementation for straight paths on 2D grid graphs}
When solving the more specific MRMP-GP-NBT(2) variant of robots moving on a 2D grid, where each robot's path is contained in a single row or column, we achieve a running time that depends only on the number of robots.

\begin{theorem}
MRMP-GP-NBT(2,0), i.e., the case where $W$ is a 2D grid and each robot's path is contained in a single row or column of $W$, has an algorithm that finds a solution, or reports that no solution exists, in time $O(n^2)$, where $n$ is the number of robots. %
\end{theorem}

In case the workspace is such a grid, the input can be given in a more compact way by representing each row/column path by its endpoints (rather than listing all the vertices on the path).
In this case, the algorithm's running time can depend only on the number of robots, rather than on the paths' lengths.

To implement Phase~1 efficiently, a plane sweep can be used to calculate the $4$ (or less) visible robots of each robot in $O(n \log n)$, and the query if a robot has a clear path to its target can be answered in $O(1)$.
Once the algorithm moves a robot $r$, w.l.o.g a robot that moves in a single row, we update the visible robots of the two visible robots of $r$ that move on the same column in constant time.
To implement Phase~2 efficiently, we can calculate all intersections of any two robots in $O(n^2)$ time.
After computing all intersections, the graph construction and all the operations on the graph can be done in linear time.
The total running time is $O(n^2)$ in the worst case.

\pagebreak

\subsection{Additional examples}
In this section, we give a few more examples demonstrating the operation of our algorithm.

\begin{figure}[H] \centering
    \begin{subfigure}{0.46\textwidth} \centering
        \includegraphics[page=1, width=1\textwidth]{figures/complex.pdf}
        \caption{} \label{}
    \end{subfigure}
    \hspace*{\fill}
    \vspace{5mm}
    \begin{subfigure}{0.23\textwidth} \centering
        \includegraphics[page=9, width=1\textwidth]{figures/complex.pdf}
        \caption{} \label{}
    \end{subfigure}
    \hspace*{\fill}
    \begin{subfigure}{0.23\textwidth} \centering
        \includegraphics[page=10, width=1\textwidth]{figures/complex.pdf}
        \caption{} \label{}
    \end{subfigure}
    \hspace*{\fill}
    \begin{subfigure}{0.23\textwidth} \centering
        \includegraphics[page=11, width=1\textwidth]{figures/complex.pdf}
        \caption{} \label{}
    \end{subfigure}
    \hspace*{\fill}
    \begin{subfigure}{0.23\textwidth} \centering
        \includegraphics[page=12, width=1\textwidth]{figures/complex.pdf}
        \caption{} \label{fig:complex_scene_contracted}
    \end{subfigure}
    \hspace*{\fill}
    \caption{\label{fig:complex_scene}
    (a) A slightly complicated instance of MRMP-GP-NBT(2), which has a single blocking cycle that requires a few untangle operations to be solved. 
    (b) The corresponding graph $G$.
    (c) $G$ after an untangle operation, $G_1$.
    (d) $G$ after two untangle operations, $G_2$.
    (e) $G$ after three untangle operations, beyond which no such operations can be done, $G_3$.}
\end{figure}

\begin{figure}[H] \centering
    \begin{subfigure}{0.23\textwidth} \centering
        \includegraphics[page=2, width=1\textwidth]{figures/complex.pdf}
        \caption{} \label{}
    \end{subfigure}
    \vspace{5mm}
    \hspace*{\fill}
    \begin{subfigure}{0.23\textwidth} \centering
        \includegraphics[page=3, width=1\textwidth]{figures/complex.pdf}
        \caption{} \label{}
    \end{subfigure}
    \vspace{5mm}
    \hspace*{\fill}
    \begin{subfigure}{0.23\textwidth} \centering
        \includegraphics[page=4, width=1\textwidth]{figures/complex.pdf}
        \caption{} \label{}
    \end{subfigure}
    \hspace*{\fill}
    \begin{subfigure}{0.23\textwidth} \centering
        \includegraphics[page=5, width=1\textwidth]{figures/complex.pdf}
        \caption{} \label{}
    \end{subfigure}
    \hspace*{\fill}
    \begin{subfigure}{0.23\textwidth} \centering
        \includegraphics[page=6, width=1\textwidth]{figures/complex.pdf}
        \caption{} \label{}
    \end{subfigure}
    \hspace*{\fill}
    \begin{subfigure}{0.23\textwidth} \centering
        \includegraphics[page=7, width=1\textwidth]{figures/complex.pdf}
        \caption{} \label{}
    \end{subfigure}
    \hspace*{\fill}
    \begin{subfigure}{0.23\textwidth} \centering
        \includegraphics[page=8, width=1\textwidth]{figures/complex.pdf}
        \caption{} \label{}
    \end{subfigure}
    \hspace*{\fill}
    \caption{\label{} After computing $G_3$ of the cycle of Figure~\ref{fig:complex_scene} as much as possible as shown in Figure~\ref{fig:complex_scene_contracted}, a motion can be calculated using Algorithm~\ref{alg:no_cycles}. The blocks of robots are $B_1=(r_0,r_1), B_2=(r_2), B_3=(r_3,r_4,r_5,r_6,r_7,r_8), B_4=(r_9), B_5=(r_{10},r_{11},r_{12},r_{13},r_{14},r_{15}), B_6=(r_{16})$. (a) The first block is chosen, $B=B_1$, and $\head(B_1)$ is moved. (b) The robots of $\nextb(B_1)$, $B_2$, are moved head to tail, and than $\head(B_1)$ is moved again. (c) The robots of $\nextb(B_2)$, $B_3$, are moved head to tail, and than $\head(B_2)$ is moved. (d,e,f) The process is repeated for the remaining blocks. (g) The cycle is solved, all robots have clear paths to their targets. }
\end{figure}

\begin{figure}[h] \centering  
  \begin{subfigure}{0.14\textwidth} \centering
    \includegraphics[page=1, width=1\textwidth]{figures/bidirection.pdf}
    \caption{} \label{}
  \end{subfigure}
  \hspace*{\fill}
  \begin{subfigure}{0.14\textwidth} \centering
    \includegraphics[page=2, width=1\textwidth]{figures/bidirection.pdf}
    \caption{} \label{}
  \end{subfigure}
  \hspace*{\fill}
  \begin{subfigure}{0.14\textwidth} \centering
    \includegraphics[page=3, width=1\textwidth]{figures/bidirection.pdf}
    \caption{} \label{}
  \end{subfigure}
  \caption{ 
  An example showing how the algorithm handles a cycle containing edges that have to be traversed in opposite directions.
  (a) A partial view of a blocking cycle where two robots, $r_0,r_1$, need to traverse the same edge in opposite directions.
  (b) A partial view of the corresponding graph $G$ that the algorithm constructs (assuming there is no scout robot in the cycle).
  (c) A partial view of $G$ after untangling the latter cycle, which enforces the constraint that $r_1$ must moved before $r_0$. } %
  \label{fig:opposite-example}
\end{figure}

\newpage

\section{Further proof details for \Cref{thm:hardness}} \label{sec:app:hardness}
In this section, we continue our proof from \Cref{sec:hardness}.

\subsection{Hardness of MRMP-GP-UNI(2,0)}

We establish the NP-hardness of MRMP-GP-UNI(2, 0) by converting the MRMP-GP-PC instance $M_s$ to an MRMP-GP-UNI(2, 0) instance $M$.
This is done by providing for each gadget $\G$ in $M_s$ a corresponding implementation on the grid, denoted by $M(\G)$, in $M$.
$M$ contains the same set of robots as $M_s$, which we call the \emph{primary} robots (for both instances). $M$ also contains \emph{secondary} robots, which are those initially located in the implemented gadgets.
Note that for these robots we reuse robot names across gadgets (even though the robots are different), but the same gadget always contains unique names.
An illustration of the complete instance $M$ is shown in \Cref{fig:hardness:full-uni}. %
Note that each gadget needs to be aligned with the paths passing through it, as indicated in each gadget's individual figure.

To prove hardness we will show that $M_s$ is solvable if and only if $M$ is solvable.
To this end, let $S_s$ and $S$ denote solutions to $M_s$ and $M$, respectively.
Based on $S_s$ (resp. $S$) we describe a converted solution $\convsol{S_s}$ (resp. $\convsol{S}$) to $M$ (resp. to $M_s$).

We now highlight a common property of both instances $M_s$ and $M$, which eases the description of the converted solutions.
Let us fix a single robot's path $\pi$.
In both $M_s$ and $M$, immediately before and after each gadget $\pi$ has a vertex that does not appear on other paths.
We call such a vertex a \emph{non-blocking vertex}.
Such a vertex can be easily created by inserting a row or column into the grid and "stretching" paths accordingly.

Before providing the implementation of each gadget, we outline how we convert between solutions of the two instances.
Both converted solutions $\convsol{S_s}$ and $\convsol{S}$ will consist of \emph{movement steps}, where after each step we have the invariant that all the primary robots are at non-blocking vertices.
Each step maintains the invariant by moving only one primary robot $r$ through a gadget to a non-blocking vertex.
Due to the invariant, $r$ cannot collide with any other primary robot during the step.
Therefore, to prove that the converted solutions are valid, it suffices to consider each gadget locally as follows.
For $\convsol{S}$ proving validity only requires verifying that each gadget is traversed in the correct order. %
As for $\convsol{S_s}$, it suffices show that collisions do not occur inside a gadget (as they cannot occur elsewhere).
We note that it is easy to verify that the invariant holds before any robot moves.

The movement steps in the converted solutions are initiated as follows.
For $\convsol{S_s}$, whenever a robot $r$ leaves a gadget vertex $\G$ in $S_s$, we initiate a movement step in which $r$ traverses corresponding gadget $M(\G)$ in $M$ to the nearest non-blocking vertex.
As part of the step, secondary robots may also be moved.
For $\convsol{S}$, since gadgets are not single vertices, we define \textit{traversal edges} whose traversal by a robot signifies that the robot has traversed the gadget.
More precisely, when a robot $r$ crosses such an edge in a gadget $M(\G)$, we move $r$ through the corresponding gadget $\G$ in $M_s$ (also to the nearest non-blocking vertex after $\G$).
In each implemented gadget, each primary robot has its own traversal edge, which we mark using a zigzag line.

We are now ready to present the gadgets.
Following the above discussion, for each gadget type we first provide the movement steps in $\convsol{S_s}$.
Then, we show that $\convsol{S}$ results in a correct traversal order in $M_s$.
For this direction of the conversion, when we say that a robot $r$ traverses a gadget in $M$, we refer to $r$ crossing its marked traversal edge.

\begin{figure}[H] \centering
    \includegraphics[page=3, scale=0.8]{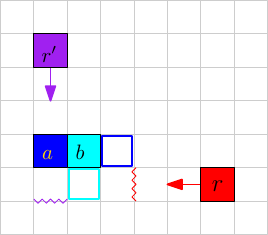}
    \caption{Implementation of a precedence gadget on the grid.}
    \label{fig:gadg:prec}
\end{figure}

\paragraph{Precedence gadget.}
Let $\G = \preced{r}{r'}$ be a precedence gadget.
Figure~\ref{fig:gadg:prec} shows $M(\G)$, the implementation of $\G$ on the grid.
The gadget can be easily rotated to accommodate paths that go left and up.
Recall that $r$ must traverse $\G$ before $r'$ in $S_s$.

In $\convsol{S_s}$ we have the corresponding steps:
(1) $r$ simply traversing $M(\G)$;
(2) $b$ and $a$ moving to their respective targets (in this order), followed by $r'$ traversing $M(\G)$.

For the other direction, observe that $b$ has to reach its target before $r'$ traverses $M(\G)$ in the solution $S$.
Therefore, $r'$ cannot traverse $M(\G)$ before $r$, as otherwise $\pi(r)$ will become blocked by $b$ at $\trg{b}$.
Therefore, we have a valid traversal order for $\G$ in $\convsol{S}$.

\begin{figure}[H] \centering
    \includegraphics[page =5, scale=0.8]{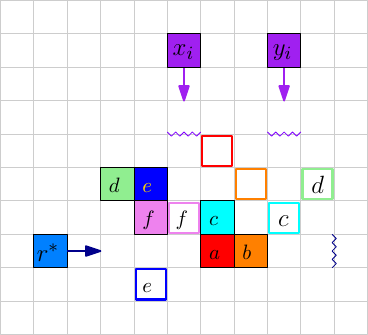}
    \caption{Implementation of $B_i$ on the grid.}
    \label{fig:gadg:befor}
\end{figure}

\paragraph{Before-constraint gadget.}
Figure~\ref{fig:gadg:befor} shows $M(B_i)$, the implementation of $B_i$ on the grid.
Recall that in $M_s$ either $x_i$ or $y_i$ traverses $B_i$ before $r^*$ does.
We assume that a solution $S_s$ to $M_s$ exists and analyze two case, which depend on whether $x_i$ or $y_i$ traverses $B_i$ before $r^*$ (the case where both do is simpler).

Case 1: Suppose that the traversal order through $B_i$ in $S_s$ is $x_i,r,y_i$.
In $\convsol{S_s}$ we have the corresponding steps:
(1) $x_i$ simply traverses $M(B_i)$ first.
(2) We then move the following robots to their targets in order: $f,e,d, b$.
Then we move $c$ one edge (keeping $\trg{c}$ free for $y_i$ to pass through) and move $a$ to its target.
These motions clear $r^*$'s path inside $M(B_i)$, so we can then have it traverse $M(B_i)$.
(3) $y_i$ traverses $M(B_i)$.

Case 2: Suppose, on the other hand, that the traversal order through $B_i$ in $S_s$ is $y_i,r^*,x_i$.
In $\convsol{S_s}$ we have the corresponding steps:
(1) After $y_i$'s traversal, (2) we move $c$ and $a$ to their respective targets (in this order).
Then, move $b$ one edge forward, which allows $r^*$ to traverse $M(B_i)$.
(3) $x_i$ traverses $M(B_i)$.

In all cases, after $r,x_i,y_i$ traverse $M(B_i)$, any secondary robots belonging to $M(B_i)$ that are still not at their respective targets can be easily moved there. %

In the other direction, let us assume for a contradiction that $r^*$ traverses $M(B_i)$ before $x_i$ and $y_i$.
It is not hard to verify using case analysis that either $b$ or $c$ must reach its target before $r^*$ traverses $M(B_i)$.
In the former case, further observation reveals $f$ must reach its target as well (before $r^*$ traverses $M(B_i)$).
Since $\trg{f}$ lies on $\pi(x_i)$ and $\trg{c}$ lies on $\pi(y_i)$, one of the latter paths must become blocked.
Therefore, either $x_i$ or $y_i$ would not be able to reach its target, which is a contradiction.
Therefore, converting $S$ yields a valid traversal order for $B_i$ in $\convsol{S}$.

\begin{figure}[H] \centering
    \includegraphics[scale=0.8]{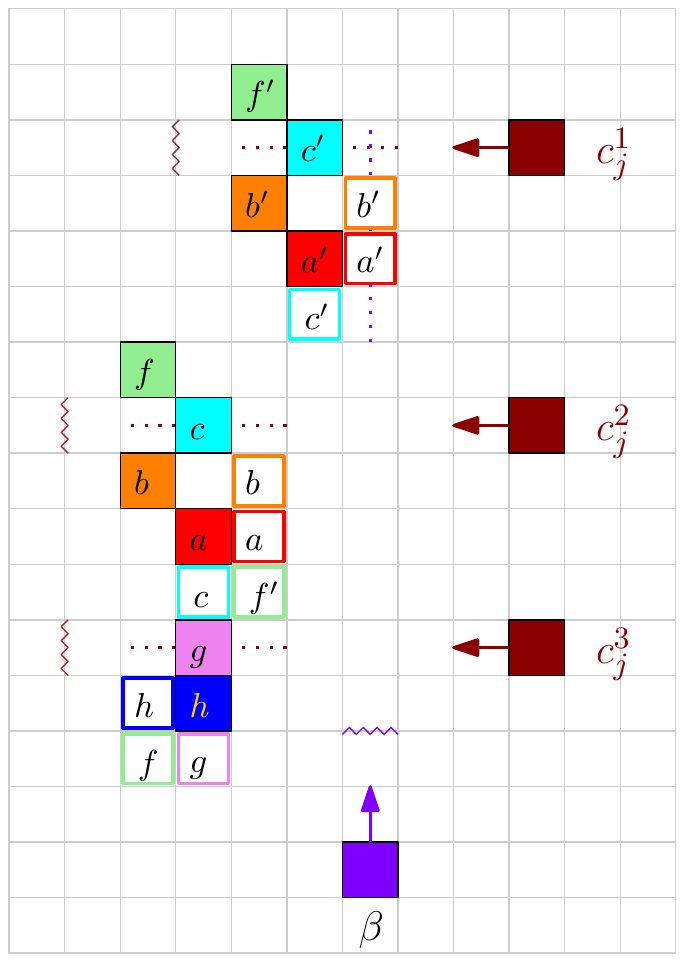}
    \caption{Implementation of $A_j$ on the grid.}
    \label{fig:gadg:after}
\end{figure}

\paragraph{After-constraint gadget.}
Figure~\ref{fig:gadg:after} shows $M(A_j)$, the implementation of $A_j$ on the grid.
Recall that in $M_s$ one of the checker robots that pass through $A_j$ must traverse $A_j$ after $\rblocker{}$ does.
Let us assume $M_s$ has a solution in which there is only one checker robot $r$ that traverses $A_j$ after $\rblocker{}$.
We assume so because the cases where multiple robots traverse $A_j$ after $\rblocker{}$ are simpler.
Note that in the cases below we can easily switch the order of the two checker robots traversing $A_j$ first, hence the first two movement steps are described together.

Case 1:
The traversal order through $A_j$ in $S_s$ is $c_j^1, c_j^2, \rblocker{}, c_j^3$.
In $\convsol{S_s}$ we have the following corresponding steps:
(1-2) Move both $c$ and $c'$ one edge.
These moves clear the paths of $c_j^2$ and $c_j^1$ through $M(A_j)$, which allows them to traverse $M(A_j)$.
(3) Then we let $\rblocker{}$ traverse $M(A_j)$.
(4) Finally, we clear the path for $c_j^3$, after which $c_j^3$ traverses $M(A_j)$, by moving the following robots to their targets in this order:
$a',c',b',f',a,c,b,f,h,g$.

Case 2:
The traversal order through $A_j$ in $S_s$ is $c_j^1, c_j^3, \rblocker{}, c_j^2$.
In $\convsol{S_s}$ we have the following corresponding steps:
(1-2) Move $c'$ and $b$ one edge.
Then move the following robots to their targets in the specified order: $f,h,g$.
These moves clear the paths of $c_j^3$ and $c_j^1$ through $M(A_j)$, which allows them to traverse $M(A_j)$.
(3) Then we let $\rblocker{}$ traverse $M(A_j)$.
(4) Finally, we clear the path for $c_j^2$, after which $c_j^2$ traverses $M(A_j)$, by moving the following robots to their targets in this order: $a',c',b',f',a,b,c$.

Case 3:
The traversal order through $A_j$ in $S_s$ is $c_j^2, c_j^3, \rblocker{}, c_j^1$.
In $\convsol{S_s}$ we have the following corresponding steps:
(1-2) Move $b'$ one edge.
Then move the following robots to their targets in the specified order: $f', a, b, c, f, h, g$.
These moves clear the paths of $c_j^3$ and $c_j^2$ through $M(A_j)$, which allows them to traverse $M(A_j)$.
(3) Then we let $\rblocker{}$ traverse $M(A_j)$.
(4) Finally, we clear the path for $c_j^1$, after which $c_j^1$ traverses $M(A_j)$, by
moving $a',b',c'$ to their targets in the specified order.

In the other direction, let us assume for a contradiction that all the checker robots of $A_j$ traverse $M(A_j)$ before $\rblocker{}$ does.
We observe the moves that must happen before the checker robots traverse $M(A_j)$.

\begin{lemma} \label{lem:hardness:before-c1-c2}
$b'$ must move one edge before both $c_j^2$ and $c_j^3$ cross their respective traversal edge in $M(A_j)$.
\end{lemma}
\begin{proof}
Before $c_j^3$ traverses $M(A_j)$, $g$ and consequently $h$ must move one edge.
Before these moves happen, $f$ must reach its target, as otherwise it will be blocked by $h$ at $\trg{h}$.
Therefore, $b$ has to move one edge before $c_j^3$ traverses $M(A_j)$, to let $f$ reach its target.
Clearly, $c$ has to move one edge before $c_j^2$ traverses $M(A_j)$.
Putting the last two facts together, simple case analysis reveals that either $a$ or $b$ must reach its target before $c_j^2$ and $c_j^3$ traverse $M(A_j)$.
Therefore, even earlier, $f'$ must advance past $\src{b'}$ in order to not get blocked by $a$ or $b$.
We get the desired result, since $b'$ must move one edge to let $f'$ perform the latter move.
\end{proof}

Now observe that $c'$ has to move one edge before $c_j^1$ traverses $M(A_j)$.
We now combine the last fact with \Cref{lem:hardness:before-c1-c2}, and obtain (by similar case analysis as for $a$ and $b$ above) that either $a'$ or $b'$ must reach its target before all the checker robots traverse $M(A_j)$.
This is a contradiction, since then $\rblocker{}$'s path through $M(A_j)$ will be blocked by the robot of $a','b$ that reaches its target.
Therefore, converting $S$ yields a valid traversal order for $A_j$ in $\convsol{S}$.

\begin{figure*}
  \centering
  \includegraphics[width=0.65\textwidth]{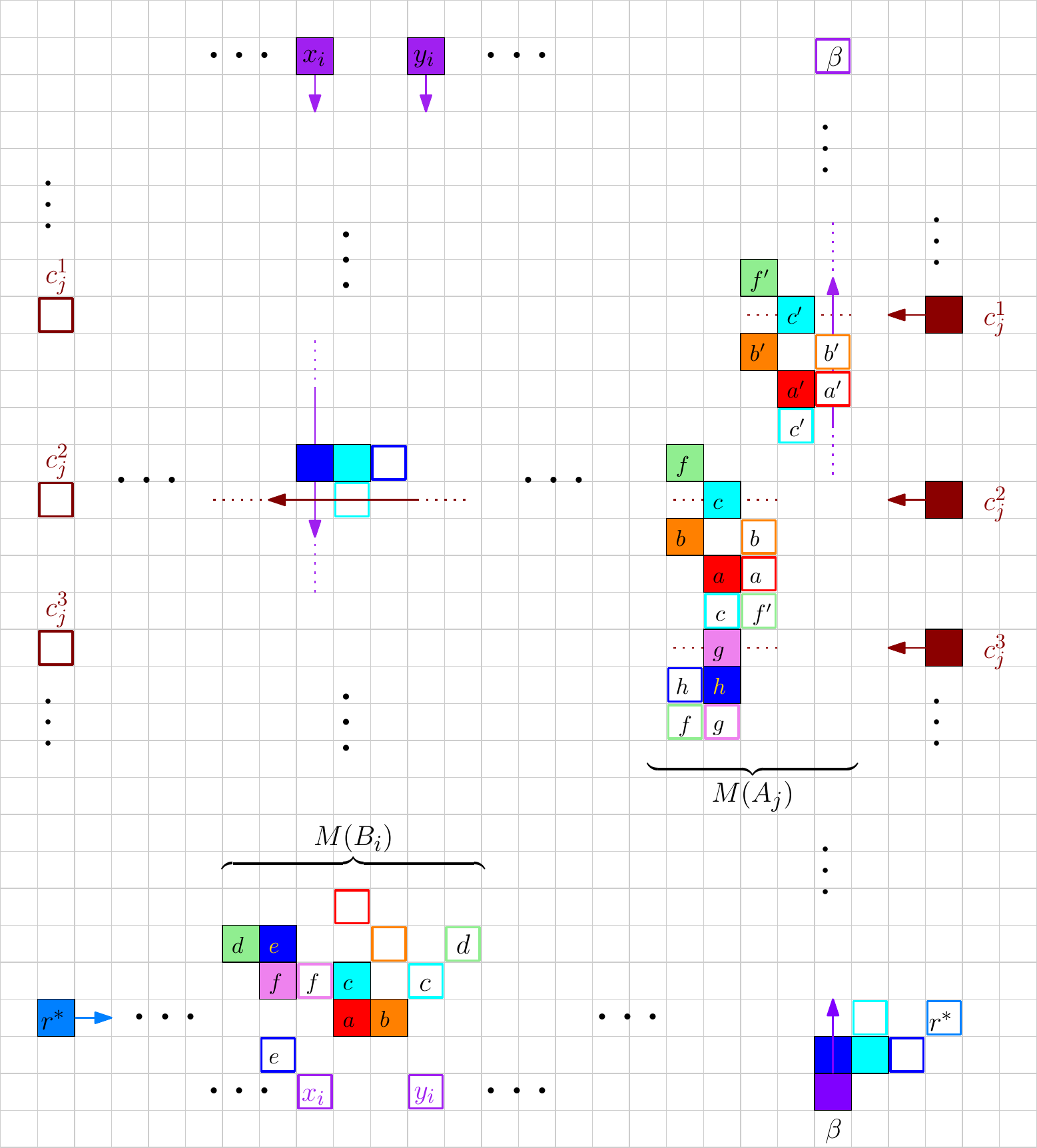}
  \caption{An illustration of $M$ showing the relative placement of the gadgets. Colored filled squares are starts and and unfilled squares are targets.}
  \label{fig:hardness:full-uni}
\end{figure*}

We conclude the conversion of $M_s$ to $M$ by observing that within each gadget of $M$ we have a vertex multiplicity of 2. Since that is also the case outside of gadgets, then $M$ is indeed an instance of MRMP-GP-UNI(2,0), as required.
Finally, it can be verified that the reduction can be completed in polynomial time.
Therefore, we obtain the following.
\begin{lemma}
MRMP-GP-UNI(2,0) is NP-complete.
\end{lemma}

\subsection{Hardness of MRMP-GP-NBT(3)}
The next stage of our reduction, towards the NP-hardness of MRMP-GP-NBT(3), is to convert $M$ to an instance $M'$ of MRMP-GP-NBT(4,1), where $W$ remains a 2D grid.
This stage requires the following mostly local changes:
Let $r,r'$ be two robots, where $\trg{r}$ is a blocking target located at vertex $v$, where $ v \in \pi(r')$.
We extend $\pi(r)$ so that it runs along $\pi(r')$ in the opposite direction, until just after passing $\src{r'}$.
\Cref{fig:block_events} illustrates this main modification, by showing how it is performed on example (a), which yields example (b).
This modification preserves the precedence constraints among the robots, with regard to the traversal order at the vertex marked by a cross.

\begin{figure}[H]
\centering
\begin{subfigure}[b]{0.35\textwidth}
        \centering
        \includegraphics[width=0.45\textwidth]{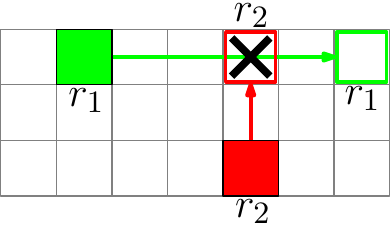} 
        \caption{}
\end{subfigure}
\begin{subfigure}{0.35\textwidth}
        \centering
        \includegraphics[width=0.45\textwidth]{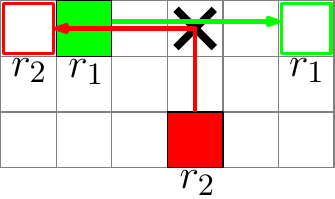} 
        \caption{}
\end{subfigure}

\caption{Two scenarios in which $r_1$ must visit the vertex marked by a cross before $r_2$.}
\label{fig:block_events}
\end{figure}

The last stage of our proof involves a conversion of $M'$ to an instance $M''$ of MRMP-GP-NBT(3).
This is done by removing vertices with VM=4 from $M'$.
It can be verified that such a change can be done whilst preserving the property that $M'$ is solvable if and only if $M''$ is solvable, as required.

\end{document}

%% file: mrmp_macros.tex
\newcommand{\poslit}{\ensuremath{R^+}}
\newcommand{\neglit}{\ensuremath{R^-}}
\newcommand{\true}{\textsc{true}}
\newcommand{\false}{\textsc{false}}
\newcommand{\litrob}{\ensuremath{r^j}}
\newcommand{\src}[1] {\ensuremath{s(#1)}}
\newcommand{\trg}[1] {\ensuremath{t(#1)}}
\newcommand{\pth}[1] {\ensuremath{\pi_{#1}}}
\newcommand{\pthdef}[2] {\ensuremath{\pi_{#1}(#2)}}
\newcommand{\gadg}{\ensuremath{\Xi}}
\newcommand{\spac}{\ensuremath{\delta}-spacious}

\newcommand{\din}{\delta_\textrm{in}}
\newcommand{\dout}{\delta_\textrm{out}}

\newcommand{\bef}{B}
\newcommand{\aft}{A}

\let\ov\overline
\newcommand{\clause}[3]{\ensuremath{(#1 \lor #2 \lor #3)}}
\newcommand{\lits}{\ensuremath{\ell_i, \ell_j, \ell_k}}
\newcommand{\cl}{\clause{\ell_i}{\ell_j}{\ell_k}}

\newcommand{\bestcost}{\mathop{d^*}}

\newcommand{\set}[1]{\ensuremath{\{#1\}}}

\definecolor{gray}{rgb}{0.35,0.35,0.35}
\definecolor{blue}{rgb}{0,0,1}
\definecolor{red}{rgb}{1,0,0}
\definecolor{orange}{rgb}{0.75, 0.4, 0}
\definecolor{green}{rgb}{0.0, 0.5, 0.0}
\newcommand{\aviv}[1]{{\color{green}\textbf{Aviv: }\sf#1}}
\newcommand{\tzvika}[1]{{\color{blue}\textbf{Tzvika: }\sf#1}}
\newcommand{\ken}[1]{{\color{orange}\textbf{Ken: }\sf#1}}
\newcommand{\michal}[1]{{\color{gray}\textbf{Michal: }\sf#1}}

\def\P{\mathcal{P}} \def\C{\mathcal{C}} \def\H{\mathcal{H}}
\def\F{\mathcal{F}} \def\U{\mathcal{U}} \def\L{\mathcal{L}}
\def\O{\mathcal{O}} \def\I{\mathcal{I}} \def\S{\mathcal{S}}
\def\G{\mathcal{G}} \def\Q{\mathcal{Q}} \def\I{\mathcal{I}}
\def\T{\mathcal{T}} \def\L{\mathcal{L}} \def\N{\mathcal{N}}
\def\V{\mathcal{V}} \def\B{\mathcal{B}} \def\D{\mathcal{D}}
\def\W{\mathcal{W}} \def\R{\mathcal{R}} \def\M{\mathcal{M}}
\def\X{\mathcal{X}} \def\A{\mathcal{A}} \def\Y{\mathcal{Y}}
\def\L{\mathcal{L}}

\def\dS{\mathbb{S}} \def\dT{\mathbb{T}} \def\dC{\mathbb{C}}
\def\dG{\mathbb{G}} \def\dD{\mathbb{D}} \def\dV{\mathbb{V}}
\def\dH{\mathbb{H}} \def\dN{\mathbb{N}} \def\dE{\mathbb{E}}
\def\dR{\mathbb{R}} \def\dM{\mathbb{M}} \def\dm{\mathbb{m}}
\def\dB{\mathbb{B}} \def\dI{\mathbb{I}} \def\dM{\mathbb{M}}
\def\dZ{\mathbb{Z}}

%% file: hardness_full.tex
\newcommand{\rblocker}{\beta}
\newcommand{\before}[2]{#2 \xleftarrow{} #1}

\newcommand{\dpreced}[3]{\P(#1,#2 \land #3)}
\newcommand{\preced}[2]{\P(#1,#2)}
\newcommand{\select}[2]{\S(#1,#2)}
\newcommand{\exec}[1]{E(#1)}

\newcommand{\convsol}[1]{\C(#1)}

\section{Hardness Results} \label{sec:hardness}

In this section, we present our hardness results, which are as follows:

\begin{theorem} \label{thm:hardness}
    The following restricted variants of MRMP-GP are NP-complete:
    \begin{enumerate}[label=(\roman*)]
        \item
        MRMP-GP-UNI(2, 0), i.e., the workspace is a 2D grid graph with a vertex multiplicity of 2. Furthermore, each agent's path is contained in a single grid row or column.

        \item  MRMP-GP-NBT(4, 1), i.e., the workspace is a 2D grid graph with a vertex multiplicity of 4. Furthermore, each agent's path contains at most one turn.
        
        \item 
        MRMP-GP-NBT(3), i.e., general graphs with vertex multiplicity of 3.
    \end{enumerate}
\end{theorem}

Note that in MRMP-GP-UNI we have unidirectional motion while blocking targets are allowed (see \Cref{sec:prob_def} for exact definitions).
In contrast, in MRMP-GP-NBT we do not allow blocking targets, but do allow bi-directional motion, i.e., robots that traverse the same path but in opposite directions.
Therefore, we can conclude that the presence of each of the elements of bi-directional motion or blocking targets by itself suffices to make MRMP-GP intractable.

Since our algorithm from \Cref{sec:main-alg} solves MRMP-GP-NBT(2), the hardness result (iii) for MRMP-GP-NBT(3) is tight.
That is, we establish a tractability frontier based on vertex multiplicity for MRMP-GP-NBT.
We identify additional tractability frontiers through the minimality of parameters as indicated in \Cref{tab:comp}, which is easy to verify.
To see that the turn number of MRMP-GP-NBT(4, 1) is minimal, we can reduce MRMP-GP-NBT(4, 0) to MRMP-GP-NBT(2).%

We prove \Cref{thm:hardness} in stages, establishing result (i), then (ii), and finally (iii).
Here we sketch the proof, focusing on (i); see \Cref{sec:app:hardness} for the rest of the details.
As it is straightforward to verify that MRMP-GP is in NP, we only discuss NP-hardness.

First, we show the NP-hardness of a variant of MRMP-GP, called \textit{MRMP-GP with Precedence Constraints (MRMP-GP-PC}).
MRMP-GP-PC has additional constraints on the order in which robots visit vertices, which we use to abstract away the details of our complete constructions while showing the functionality of their gadgets.
As part of the sketch, we illustrate how to realize the instance constructed by the reduction to MRMP-GP-PC, denoted by $M_s$, on the 2D grid.
In our full proof, we incrementally convert $M_s$, to instances that are equally hard to solve.
That is, each subsequent instance $M'$ we describe is solvable if and only $M_s$ is solvable.

The conversion steps of $M_s$ are as follows.
We first convert $M_s$ to an instance $M$ of MRMP-GP-UNI(2,0).
To realize the gadgets in $M$, we use blocking targets.
In the next stage, we convert $M$ to $M'$, in which we eliminate all the blocking targets.
Such targets are replaced by paths going in opposite directions. %
This stage requires the following mostly local changes:
Let $r,r'$ be two robots, where $\trg{r}$ is a blocking target located at vertex $v$, where $ v \in \pi(r')$.
We extend $\pi(r)$ so that it runs along $\pi(r')$ in the opposite direction. %
The change preserves the constraint that $r$ must visit $v$ before $r'$ in a valid solution, which we use in our gadgets.
The changes increase the VM of $M'$ to 4 and make it an instance of MRMP-GP-NBT(4,1).
Lastly, by carefully removing vertices with VM=4 from $M'$ (which makes it no longer a 2D grid) we get an instance $M''$ of MRMP-GP-NBT(3).

We now begin the sketch, focusing on the hardness of MRMP-GP-UNI(2, 0).
To establish the hardness of MRMP-GP-PC, we introduce a problem called Pivot Scheduling and prove that it is NP-hard.
Next, we reduce Pivot Scheduling to MRMP-GP-PC.

\paragraph{Pivot Scheduling.}
An instance has the form $(V, \C)$, where $V$ is a set of jobs that come in pairs and $\C$ is a set of ordering constraints:
Let $V= \set{x_1, y_1, \ldots, x_n, y_n}$ be a set of $2n$ distinct jobs.
Let $\C = \{C_1, \ldots, C_m\}$, where each $C_j \in V^3$ is a triplet.
The problem\footnote{Pivot Scheduling can be seen as a special case of generalized AND/OR Scheduling, which is NP-hard~\cite{DBLP:journals/siamcomp/MohringSS04}.}
is to determine whether $V$ can be partitioned into a before-set $\bef$ and an after-set $\aft$ (i.e., $\aft \cap \bef = \emptyset$ and $\aft \cup \bef = V$) such that the following constraints are satisfied: (i) for each pair $x_i, y_i$, we have either $x_i \in \bef$ or $y_i \in \bef$ and (ii) for each $C_j \in \C$, one of the jobs in $C_j$ must be in $\aft$, i.e., $C_j \cap \aft \neq \emptyset$.

We call the former constraints \textit{before-constraints} and the latter constraints \textit{after-constraints}.
Intuitively, the before/after constraints implicitly imply the existence of a distinguished \textit{pivot job} with respect to which the input jobs must be ordered.
To be precise, $\bef$ and $\aft$ respectively correspond to the jobs that come before and after the pivot job.

\begin{lemma}
Pivot Scheduling is NP-hard.
\end{lemma}

\begin{proof} %
We present a reduction from 3SAT, the problem of
deciding the satisfiability of a formula in conjunctive normal form with three literals in each clause.
Given a 3SAT formula $\phi$, we define a corresponding instance of Pivot Scheduling $S(\phi) = (V, \C)$ as follows:
Each variable and its negation form a pair of jobs in $V$, i.e., $V = \bigcup_{x \in \phi} \set{x,\ov{x}}$.
As for the set of after constraints $\C$, each clause of $\phi$ defines a constraint, which contains the literals in the clause.
For simplicity, we use the same notation for a job and its corresponding literal.%

To complete the reduction, we define the correspondence between a satisfying assignment $\A$ for $\phi$ and a valid partition of $V$:
A variable $x$ is assigned "true" if and only if the job $x$ is in the after-set, i.e., $x \in \aft$.
In other words, given a satisfying assignment for $\phi$, $\bef$ and $\aft$ correspond to literals assigned "false" and "true", respectively.
It is straightforward to verify that the resulting partition satisfies the constraints if and only if $\A$ satisfies $\phi$.
\end{proof}

\paragraph{Hardness of MRMP-GP-PC.}
Let us now define MRMP-GP-PC.
The input is the same as MRMP-GP except that we also have special vertices, which we will use as gadgets.
Each \emph{gadget vertex} must be traversed (i.e., visited) by the robots passing through it in a certain \emph{traversal order} (the exact constraints used will be specified in the reduction).
A solution to MRMP-GP-PC is the same as for MRMP-GP with the additional requirement that gadget vertices must be visited according to their respective traversal order.

We proceed to the reduction.
Given an instance of Pivot Scheduling $I=(V,\C)$, we construct a corresponding MRMP-GP-PC instance $M_s$. %
We represent each job in $V$ by a corresponding \textit{job robot} in $M_s$.
We also represent the implicit pivot job by the \textit{pivot robot} $r^*$.
To simplify notation, we use the same name for a job and its corresponding robot.
Another robot is $\rblocker{}$, which can be thought of as a robot continuing the journey of $r^*$.
Lastly, for each constraint $C_j = \set{z_1, z_2, z_3} \in V^3$, we have three corresponding \textit{checker} robots $c^1_j, c^2_j, c^3_j$, which we use to emulate after-constraints.

We now discuss the gadgets in $M_s$. %
For each $i \in [n]$ we have a \emph{before-constraint gadget}, denoted by $B_i$, which appears on the paths of $r^*$, $x_i$ and $y_i$.
The traversal order for $B_i$ is defined as having either $x_i$ or $y_i$ traverse $B_i$ before $r^*$ does.
For each $j \in [m]$ we have an \emph{after-constraint gadget}, denoted by $A_j$, which appears on the paths of $\rblocker{}$ and the checker robots $c^1_j, c^2_j, c^3_j$.
The traversal order for $A_j$ is defined as having one of $c^1_j, c^2_j, c^3_j$ traverse $A_j$ after $\rblocker{}$ does.
The last gadget type we use is the \emph{precedence gadget}, denoted by $\preced{r}{r'}$, which must be visited first by $r$ and then by $r'$, where $r$ and $r'$ are arbitrary robots.
The particular instances of this gadget that we use are specified below.

We now describe the order of gadgets along the paths.
The path $\pi(r^*)$ first passes through all before-constraint gadgets (in arbitrary order) and then passes through the precedence gadget $\G \coloneqq \preced{r^*}{\rblocker{}}$.
The path $\pi(\rblocker{})$ first passes through $\G$, and then through all after-constraint gadgets (in arbitrary order).
Since $\G$ is the last gadget along $\pi(r^*)$ and the first gadget along $\pi(\rblocker{})$, $\rblocker{}$ can essentially only start moving after $r^*$ reaches its target.
Now let $c^{\ell}_j$ be a checker robot, which corresponds to the job $z$ in the constraint $C_j$.
The robot $c^{\ell}_j$ first passes through $A_j$ and then through the precedence gadget $\preced{c^{\ell}_j}{z}$.
This means that before a job robot enters its respective $B_i$ gadget, all the checker robots corresponding to the job must first traverse their respective $A_j$ gadget.

\paragraph{Robots and path placements.}
An example of $M_s$ is shown in \Cref{fig:hardness_overall}, which has a dual purpose of illustrating the conversion of $M_s$ to the grid instance $M$.
The figure should be interpreted as follows.
Each (long and colored) arrow represents a path lying in a single grid row/column and the rectangles are gadgets.
For $M$, each rectangle indicates the placement of a gadget, which contains a constant number of additional robots (not shown).
For $M_s$, each gadget is a vertex, so a robot's path can be thought of as going off the grid for one vertex to visit the gadget.
Note that vertex multiplicity outside of gadgets is indeed 2.

We realize $M$ on the 2D grid as follows.
The top row of $M$ initially contains the job robots $x_1, y_1, \ldots, x_n, y_n$, ordered left to right, which have to move down to the bottom row of $M$.
The rightmost column of $M$ initially contains the checker robots $c^1_1, c^2_1, c^3_1, \ldots, c^1_m, c^2_m, c^3_m$, ordered top to bottom, which have to move left to the leftmost column of $M$.
The pivot robot $r^*$ is initially located near the bottom row and has to go from the leftmost column to the rightmost column of $M$.
The robot $\rblocker$ is initially located near the rightmost column and has to go from the bottom row to the top row of $M$.

\begin{figure}[t]
  \centering
  \includegraphics[page=1, width=0.28\textwidth]{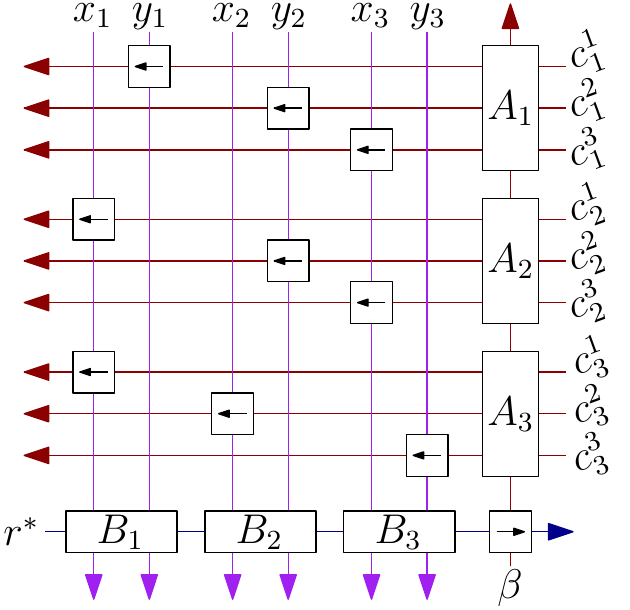}
  \caption{An MRMP-GP-PC instance corresponding to Pivot Scheduling with $V=\set{x_1,y_1,\ldots,x_3,y_3}$ and $\C=\set{\set{y_1,y_2,x_3}, \set{x_1,y_2,x_3}, \set{x_1,x_2,y_3}}$.
The robots' paths are shown as long arrows.
Precedence gadgets are shown as squares containing a short arrow, which is oriented along the path of the robot that needs to traverse the gadget first.
}
   \label{fig:hardness_overall}
\end{figure}

We now turn to the correctness of the reduction.
We have the following correspondence between a partition $(B,A)$ for $I$ and a solution to $M_s$: A job of an $x_i, y_i$ pair is in the before-set $B$ if and only if the corresponding job robot traverses $B_i$ before $r^*$ does.
Now let $(B', A')$ be a partition obtained from a solution to $M_s$, as just defined.
Clearly, by the definition of the $B_i$'s, $B'$ satisfies the before-constraints of $I$.
We now prove that after constraints are satisfied by $A'$:

\begin{lemma}
Let $C_j$ be an after-constraint of $I$ and let $R(C_j)$ denote the corresponding job robots in $M_s$.
Then one of the robots of $R(C_j)$ traverses its respective before-constraint gadget after $r^*$.
\end{lemma}
\begin{proof}
Assume for a contradiction that each robot in $R(C_j)$ traverses its respective before-constraint gadget before $r^*$ does.
Let $r$ be the last robot of $R(C_j)$ to do so and let $B_i$ be the gadget it traverses.
We now examine the time step in which $r$ is at $B_i$.
At this point, $r^*$ is located to the left of $B_i$ while all the checker robots of $C_j$ have already traversed the gadget $A_j$ due to the precedence gadgets.
Consequently, $\rblocker{}$ must have already traversed $A_j$, as it cannot be the last robot to traverse the $A_j$, by its definition.
In particular, $\rblocker{}$ already traversed the precedence gadget $\G = \preced{r^*}{\rblocker}$ (bottom right in \Cref{fig:hardness_overall}).
This is a contradiction since $r^*$ must traverse $\G$ before $\rblocker$.
\end{proof}

Given a valid job partition, it is easily verified that the correspondence above directly lends itself to an ordering of the robots of $M_s$ by which they can move to their targets one by one. Therefore, we obtain the following:

\begin{lemma}
MRMP-GP-PC is NP-complete.
\end{lemma}

The complete realization of $M$ on the 2D grid as well as the subsequent steps proving \Cref{thm:hardness} are given in \Cref{sec:app:hardness}.